\documentclass[10pt]{article} % For LaTeX2e
\usepackage[accepted]{tmlr}
\usepackage{amsmath,amsfonts,bm}

\def\eqref#1{equation~\ref{#1}}
\def\1{\bm{1}}

\DeclareMathAlphabet{\mathsfit}{\encodingdefault}{\sfdefault}{m}{sl}
\SetMathAlphabet{\mathsfit}{bold}{\encodingdefault}{\sfdefault}{bx}{n}

\newcommand{\E}{\mathbb{E}}

\newcommand{\R}{\mathbb{R}}

\DeclareMathOperator*{\argmax}{arg\,max}

\usepackage{amsthm}
\newtheorem{thm}{Theorem}[section]
\newtheorem{lem}[thm]{Lemma}
\newtheorem{prop}[thm]{Proposition}

\newtheorem{cor}[thm]{Corollary}

\newtheorem{rmk}[thm]{Remark}

\newcommand{\mbf}[1]{\mathbf{#1}}

\newcommand{\norm}[1]{\left\lVert#1\right\rVert}

\newcommand{\T}[1]{#1^\intercal}

\newcommand{\sprod}[1]{\left<#1\right>}

\usepackage{hyperref}
  \hypersetup{
    final=true,
    plainpages=false,
    pdfstartview=FitV,
    pdftoolbar=true,
    pdfmenubar=true,
    bookmarksopen=true,
    bookmarksnumbered=true,
    breaklinks=true,
    linktocpage=true,
    colorlinks=true,
    linkcolor=red, %black,
    urlcolor=iris, %black,
    citecolor=darkcerulean, %green, %black,
    anchorcolor=black
  }

\usepackage{amsmath}
\usepackage{amsfonts}
\usepackage{bm}
\usepackage{amssymb}
\usepackage{mathtools}
\usepackage{amsthm}
\usepackage{dsfont}
\usepackage{bbm} % column vector 1
\usepackage{stmaryrd} % llbrackets
\usepackage{enumitem}

\usepackage{mleftright} % for '\mleft' and '\mright' macros

\newcommand{\vertiii}[1]{{\left\vert\kern-0.25ex\left\vert\kern-0.25ex\left\vert #1 
    \right\vert\kern-0.25ex\right\vert\kern-0.25ex\right\vert}}
\usepackage{stmaryrd}

\usepackage{dirtytalk}
\MakeRobust{\say} % to use \say in captions

\usepackage{makecell}

\usepackage{wrapfig}

\usepackage{mathtools} % e.g. for declaring paired limiter
\DeclarePairedDelimiterX{\infdivx}[2]{(}{)}{%
  #1\;\delimsize\|\;#2%
}

\usepackage{url}            % simple URL typesetting
\usepackage{nicefrac}       % compact symbols for 1/2, etc.
\usepackage{microtype}      % microtypography
\usepackage{subfigure}

\usepackage{algorithm2e}
\RestyleAlgo{ruled}

\usepackage{booktabs}
\usepackage{multirow}
\usepackage{graphicx}
\usepackage{caption} 
\usepackage{mdframed}
\newmdtheoremenv[topline=false, bottomline=false, leftline=false, rightline=false, backgroundcolor=aliceblue,%
innertopmargin=\topskip, splittopskip=\topskip, skipbelow=\baselineskip, skipabove=\baselineskip]{boxthm}{Theorem}[section]
\newmdtheoremenv[topline=false, bottomline=false, leftline=false, rightline=false, backgroundcolor=aliceblue,%
innertopmargin=\topskip, splittopskip=\topskip, skipbelow=\baselineskip, skipabove=\baselineskip]{boxprop}[boxthm]{Proposition}
\newmdtheoremenv[topline=false, bottomline=false, leftline=false, rightline=false, backgroundcolor=aliceblue,%
innertopmargin=\topskip, splittopskip=\topskip, skipbelow=\baselineskip, skipabove=\baselineskip]{boxexample}[boxthm]{Example}
\newmdtheoremenv[topline=false, bottomline=false, leftline=false, rightline=false, backgroundcolor=aliceblue,%
innertopmargin=\topskip, splittopskip=\topskip, skipbelow=\baselineskip, skipabove=\baselineskip]{boxcor}[boxthm]{Corollary}
\newmdtheoremenv[topline=false, bottomline=false, leftline=false, rightline=false, backgroundcolor=aliceblue,%
innertopmargin=\topskip, splittopskip=\topskip, skipbelow=\baselineskip, skipabove=\baselineskip]{boxlem}[boxthm]{Lemma}
\newmdtheoremenv[topline=false, bottomline=false, leftline=false, rightline=false, backgroundcolor=aliceblue,%
innertopmargin=\topskip, splittopskip=\topskip, skipbelow=\baselineskip, skipabove=\baselineskip]{boxdef}[boxthm]{Definition}

\usepackage{tikz}
\usepackage{varwidth}

\definecolor{rightblue}{RGB}{76,114,176} %#4C72B0
\definecolor{rightorange}{RGB}{221,132,82} %#DD8452
\definecolor{aliceblue}{rgb}{0.94, 0.97, 1.0} % boxes
\definecolor{darkcerulean}{rgb}{0.03, 0.27, 0.49} % citation
\definecolor{iris}{rgb}{0.35, 0.31, 0.81} % URL
\definecolor{carmine}{rgb}{0.59, 0.0, 0.09} % toy example
\definecolor{green(munsell)}{rgb}{0.0, 0.66, 0.47} % toy example
\definecolor{celadon}{rgb}{0.67, 0.88, 0.69} % table

\let\classAND\AND
\let\AND\relax
\usepackage{algorithmic}
\let\AND\classAND
\AtBeginEnvironment{algorithmic}{\let\AND\algoAND}
\usepackage{bbding}
\usepackage[utf8]{inputenc} % allow utf-8 input
\usepackage[T1]{fontenc}    % use 8-bit T1 fonts

\usepackage{hyperref}       % hyperlinks
\usepackage{url}            % simple URL typesetting
\usepackage{booktabs}       % professional-quality tables
\usepackage{amsfonts}       % blackboard math symbols
\usepackage{nicefrac}       % compact symbols for 1/2, etc.
\usepackage{microtype}      % microtypography
\usepackage{xcolor}         % colors

\usepackage{amsmath}
\usepackage{amssymb}
\usepackage{mathtools}
\usepackage{amsthm}
\usepackage{dsfont}
\usepackage{subfigure}
\usepackage{bbding}
\usepackage{xcolor} 
\usepackage{caption}

\theoremstyle{plain}

\theoremstyle{definition}

\theoremstyle{remark}

\title{Leveraging Gradients for Unsupervised Accuracy Estimation under Distribution Shift}

\author{\name Renchunzi Xie \email renchunzi.xie@ntu.edu.sg \\
      \addr College of Computing and Data Science\\
      Nanyang Technological University
      \AND
      \name Ambroise Odonnat \email ambroise.odonnat@gmail.com \\
      \addr Huawei Noah's Ark Lab, Inria\textsuperscript{$\diamond$}\\
      Paris, France
      \AND
      \name Vasilii Feofanov \email vasilii.feofanov@gmail.com\\
      \addr Huawei Noah's Ark Lab \\
      Paris, France
      \AND
      \name Ievgen Redko \email ievgen.redko@gmail.com\\
      \addr Huawei Noah's Ark Lab \\
      Paris, France
      \AND
      \name Jianfeng Zhang \email ievgen.redko@gmail.com\\
      \addr Huawei Noah’s Ark Lab \\
      Shenzhen, China
      \AND
      \name Bo An\textsuperscript{*} \email boan@ntu.edu.sg\\
      \addr  College of Computing and Data Science\\
      Nanyang Technological University
     }

\def\month{04}  % Insert correct month for camera-ready version
\def\year{2025} % Insert correct year for camera-ready version
\def\openreview{\url{https://openreview.net/forum?id=FIWHRSuoos}} % Insert correct link to OpenReview for camera-ready version

\usepackage[para]{footmisc}

\makeatletter
\long\def\@makefntext#1{\leavevmode
  \@makefnmark\nobreak
  #1%
}
\makeatother
\begin{document}
\def\thefootnote{*}\footnotetext{Corresponding Authors.} 
\def\thefootnote{$\diamond$}\footnotetext{Univ. Rennes 2, CNRS, IRISA}

\renewcommand{\thefootnote}{\arabic{footnote}}
\setcounter{footnote}{0}
\maketitle

\begin{abstract}
 Estimating the test performance of a model, possibly under distribution shift, without having access to the ground-truth labels is a challenging, yet very important problem for the safe deployment of machine learning algorithms in the wild. Existing works mostly rely on information from either the outputs or the extracted features of neural networks to estimate a score that correlates with the ground-truth test accuracy. In this paper, we investigate -- both empirically and theoretically -- how the information provided by the gradients can be predictive of the ground-truth test accuracy even under distribution shifts. More specifically, we use the norm of classification-layer gradients, backpropagated from the cross-entropy loss after only one gradient step over test data. Our intuition is that these gradients should be of higher magnitude when the model generalizes poorly. We provide the theoretical insights behind our approach and the key ingredients that ensure its empirical success. Extensive experiments conducted with various architectures on diverse distribution shifts demonstrate that our method significantly outperforms current state-of-the-art approaches. The code is available at \url{https://github.com/Renchunzi-Xie/GdScore}.
\end{abstract}

\section{Introduction}
Deploying machine learning models in the real world is often subject to a distribution shift between training and test data. Such a shift may significantly degrade the model's performance at test time \citep{quinonero2008dataset,geirhos2018imagenet,koh2021wilds} and lead to high risks related to AI satefy \citep{deng2021labels}. To alleviate this problem, a common practice is to monitor the model performance regularly by collecting the ground truth of a subset of the current test dataset \citep{lu2023characterizing}. However, this is usually time-consuming and expensive, highlighting the need for unsupervised methods to assess the test performance of models under distribution shift, commonly known as \textit{unsupervised accuracy estimation}. 

\paragraph{Limitations of current approaches.} Current studies mainly focus on outputs or feature representation to derive a test error estimation score. Such a score can represent the calibrated test error or the distribution discrepancy between training and test datasets \citep{hendrycks2016baseline,guillory2021predicting, garg2022leveraging, deng2021labels, yu2022predicting, lu2023characterizing}. For instance, \citet{hendrycks2016baseline} considered the average maximum softmax score of the test samples as the estimated error. Similarly, \citet{garg2022leveraging} proposed to learn a confidence threshold from the training distribution. \citet{deng2021labels} quantified the distribution difference between training and test datasets in the feature space, while \citet{yu2022predicting} gauges the distribution gap at the parameter level. %Although insightful, the current body of literature overlooks another important aspect of test accuracy estimation even though gradients are known to correlate strongly with the generalization error of the deep neural networks \citep{li2019generalization,an2020can}. 
Although insightful, the current body of literature overlooks a potential tool for test accuracy estimation, namely the gradients that are known to correlate strongly with the generalization error of the deep neural networks \citep{li2019generalization,an2020can}. 
\iffalse
Thus, in this work, we seek to answer the following question:
\begin{quote}
    \centering\textit{Are gradients predictive of test errors under distribution shift?}
\end{quote}
\fi

% Although insightful, existing methods are either prone to overfitting \citep{wei2022mitigating}, have a high computational cost \citep{yu2022predicting}, or require strong assumptions about the underlying distribution shift \citep{lu2023characterizing}. To address those issues, we propose a lightweight, robust, and strongly predictive OOD error estimation method relying on gradients that were often overlooked in prior work. Inspired by the strong correlation of the gradients with the generalization error of deep neural networks \citep{li2019generalization,an2020can}, we seek to answer the following question:

% \begin{quote}
%     \centering\textit{Are gradients predictive of OOD errors?}
% \end{quote}
\paragraph{Why do gradients matter?} %The role of gradients in generalization has attracted increasing attention, and many algorithms are designed based on them \citep{shi2021gradient, zhou2020towards, mansilla2021domain, zhao2022penalizing, guiroy2019towards}. 

Recently, increasing attention has been paid to the intrinsic properties of gradients to design learning algorithms for meta-learning \citep{finnModelAgnosticMetaLearningFast2017}, domain generalization \citep{shi2021gradient,mansilla2021domain} or to improve optimization of DNNs \cite{zhou2020towards,zhao2022penalizing} by relying on them.
% \ievgen{The role of gradients has recently received in machine learning increasing attention , with algorithms in meta-learning \citep{finnModelAgnosticMetaLearningFast2017}, domain generalization \citep{shi2021gradient,mansilla2021domain} and improving optimization of DNNs \cite{zhou2020towards,zhao2022penalizing} relying on them.}
% \citet{mansilla2021domain} proposed to clip the conflicting gradients emerging in domain shift scenarios and to promote gradient agreement across multiple domains via the gradient agreement strategy. 
In domain shift scenarios, \citet{mansilla2021domain} proposed to clip the conflicting gradients and introduced a strategy to promote gradient agreement across multiple domains. 
Meanwhile, \citet{zhao2022penalizing} designed a gradient-based regularization term to make the optimizer find flat minima. In the field of OOD detection, 
% In a more general field, 
\citet{huang2021importance} introduced a gradient-based function to detect out-of-distribution (OOD) samples (relation to our method is compared in Appendix~\ref{app:comparison_grad_norm}).
% \vasilii{the approach shares some similarities with ours, so we explicitly compare with it in Appendix~\ref{app:comparison_grad_norm}}).
%However, despite their empirical success, how gradients affect generalization performance is still an open question. 
% \ievgen{Despite this, prior work overlooked the potential that gradient-based scores can carry for the tasks of unsupervised accuracy estimation, although the works mentioned above clearly showed the potential of gradients to tackle both the learning problems in OOD and IID settings suggesting their possible correlation with accuracy. This motivates us to ask: }
Although the works mentioned above showed the potential of gradients to tackle the learning problems both in OOD and in-distribution (ID) settings, it is still unclear how unsupervised gradient-based scores can correlate with the test accuracy and be used to estimate it. This motivates us to ask:
% \vasilii{
% Although the works mentioned above clearly showed the potential of gradients to tackle the learning problems both in OOD and in-distribution (ID) settings, they have not studied yet 
% how unsupervised gradient-based scores can possibly correlate with the test accuracy and be used to estimate it. This motivates us to ask:}
\begin{quote}
    \centering\textit{Are gradients predictive of test accuracy under distribution shift?}
\end{quote}
In this paper, we shed new light on this open question: we surprisingly observe that there exists a strong linear relationship between gradient norm and test accuracy under distribution shift. %To explore this phenomenon, we theoretically certify that well-calibrated gradients include information directly related to the test accuracy. 
Moreover, our theoretical analysis shows that the gradient norm conveys information on the generalization capacity of a well-calibrated model. In a nutshell, this work provides direct evidence that gradient-based information correlates with generalization performance, which paves the way to a better understanding of how neural networks generalize across unseen domains. %might help us identify how and when neural networks make errors as well as generalizing across unseen domains. 

\iffalse
\paragraph{Our contributions.} We tackle the above question by transposing the use of the gradient norm for generalization accuracy prediction to the distribution-shift setting. We hypothesize that the model requires a large magnitude gradient step if it cannot generalize well on the test dataset from unseen domains at hand. To quantify the magnitude of gradients, we propose a simple yet efficient gradient-based statistic, \textit{\textsc{GdScore}}, which employs the vector norm of gradients backpropagated from a standard cross-entropy loss on the test samples. To avoid the need for ground-truth labels, we propose a pseudo-labeling strategy that seeks to benefit from both correct and incorrect predictions. Ultimately, we demonstrate that the simple one-step gradient norm of the classification layer strongly correlates with the generalization performance under diverse distribution shifts acting as a strong and lightweight proxy for the latter.

The main contributions of our paper are summarized as follows:
\fi
\paragraph{Our contributions.} We hypothesize that the model requires a gradient step of large magnitude when it fails to generalize well on test data from unseen domains. To quantify the magnitude of gradients, we propose a simple yet efficient gradient-based statistic, \textit{\textsc{GdScore}}, which employs the norm of the gradients backpropagated from a standard cross-entropy loss on the test samples. To avoid the need for ground-truth labels, we propose a pseudo-labeling strategy that benefits from both correct and incorrect predictions. We demonstrate that the norm of this one-step gradient of the classification layer strongly correlates with the generalization performance under diverse distribution shifts, acting as a strong and lightweight proxy for the latter. The main contributions of our paper are summarized as follows:
\begin{enumerate}
    \item We first provide several theoretical insights showing that correct pseudo-labeling and gradient norm have a direct impact on the test error estimation. This is achieved by looking at the analytical expression of the gradient after one backpropagation step over the pre-trained model on test data under distribution shift and by upper-bounding the target out-of-distribution risk.  
    \item Based on these theoretical insights, we propose the \textsc{GdScore}, which gauges the magnitude of the classification-layer gradients and presents a strong correlation with test accuracy. Our method does not require access to either test labels or training datasets and only needs one step of backpropagation which makes it particularly lightweight in terms of computational efficiency compared to other self-training methods.
    \item We demonstrate the superiority of \textsc{GdScore} with a large-scale empirical evaluation. We achieve new state-of-the-art results on 11 benchmarks across diverse distribution shifts compared to 8 competitors, while being faster than the previous best baseline.
\end{enumerate}
\paragraph{Organization of the paper.} The rest of the paper is organized as follows. Section~\ref{section:background} presents the necessary background on the problem at hand. In Section~\ref{section:contributions}, we derive the theoretical insights that motivate the \textsc{GdScore} introduced afterward. Section~\ref{section:experiments} is devoted to extensive empirical evaluation of our method, while the ablation study is deferred to Section~\ref{section_ablation}. %To avoid confusing the reader, we explicit the main differences between our method and the current work \citet{huang2021importance} in Section~\ref{sec:comparison_gradnorm}. 
Finally, Section \ref{conclusion} concludes our work. 

\section{Related Work}
\label{app:related_work}
\paragraph{Unsupervised accuracy estimation.} Unsupervised accuracy estimation is a vital topic in practical applications due to frequent distribution shifts and the unavailability of ground-truth labels for test samples. To comprehensively understand this field, we introduce two main existing settings related to this topic.
\begin{enumerate}
    \item Some works aim to estimate the test accuracy or gauge the accuracy discrepancy between the training and the test set only via the training data \citep{corneanu2020computing, jiang2018predicting, neyshabur2017exploring, unterthiner2020predicting, yak2019towards, martin2020heavy}. For example, the model-architecture-based algorithm \citep{corneanu2020computing} derives plenty of persistent topology properties from the training data, which can identify when the model learns to generalize to unseen datasets. However, those algorithms are deployed under the assumption that the training and the test data are from the same distribution, which means they are vulnerable to distribution shifts.
    \item Our work belongs to the second setting, which aims to estimate the classification accuracy of a specific test dataset during evaluation using unlabeled test samples and/or labeled training datasets. The main research direction is to explore the negative relationship between the distribution discrepancy and model performance from the space of features \citep{deng2021labels}, parameters \citep{yu2022predicting} and labels \citep{lu2023characterizing}. Another popular direction is to design an estimation score via the softmax outputs of the test samples \citep{guillory2021predicting, jiang2021assessing, guillory2021predicting, garg2022leveraging}, which heavily relies on model calibration. Some works also learn from the field of unsupervised learning, such as agreement across multiple classifiers \citep{jiang2021assessing, madani2004co, platanios2016estimating, platanios2017estimating} and image rotation \citep{deng2021does}. In addition, the property of the test datasets presented during evaluation has also been studied recently \citep{xie2023importance}. To the best of our knowledge, our work is the first to study the linear relationship between the gradients and model performance.
\end{enumerate}
\textbf{Gradients in generalization.} The role of gradients in generalization has attracted increasing attention recently. To gauge the generalization performance of the hypothesis learned from the training data on unseen samples, known as the out-of-sample error \citep{hardt2016train, london2017pac, rivasplata2018pac}, many studies try to provide a tight upper bound for generalization error from the view of gradient descent theoretically, indicating that gradients correlate with the discrepancy between the empirical loss and the population loss \citep{li2019generalization, chatterjee2020coherent, negrea2019information, an2020can}. However, those works assume that seen to unseen data are from the identical distribution, while unsupervised accuracy estimation discusses a more complex and realistic issue that they come from different distributions. Under distribution shift, gradients are also explored. For example, \citet{mansilla2021domain} clips the conflicting gradients emerging in domain shift scenarios and promotes gradient agreement across multiple domains via the gradient agreement strategy. Similarly, \citet{zhao2022penalizing} designs a gradient-based regularization term to make the optimizer find the flat minima. In out-of-distribution (OOD) detection which goal is to determine whether a given sample is in-distribution (ID) or out-of-distribution \citep{hendrycks2016baseline, hendrycks2018deep, liu2020energy, yang2021generalized, liang2017enhancing}, \citep{huang2021importance} finds that ID data usually have higher gradient magnitude than OOD data from current source distribution to a uniform distribution. However, despite their empirical success, the relationship between gradients and generalization is still unclear.

\section{Background} \label{section:background}

\paragraph{Problem setup.}
We consider a $K$-class classification task with the input space $\mathcal{X}\!\subset\!\mathbb{R}^D$ and the label set $\mathcal{Y} = \{1, \ldots, K\}$. Our learning model is a neural network with trainable parameters ${\bm{\theta}} \in \mathbb{R}^p$ that maps from the input space to the label space $f_{\bm{\theta}}: \mathcal{X} \rightarrow \mathbb{R}^K$. We view the network as a combination of a complex feature extractor $f_{\mbf{g}}$ and a linear classification layer $f_{\bm{\omega}}$, where $\mbf{g}$ and $\bm{\omega}=(\mbf{w}_k)_{k=1}^K$ denote their corresponding parameters. Given a training example $\mbf{x}_i$, the feedforward process can be expressed as:
\begin{equation}
\label{feedforward}
    f_{\bm{\theta}}(\mbf{x}_i) = f_{\bm{\omega}}(f_{\mbf{g}}(\mbf{x}_i)).
\end{equation}
Let $\mbf{y}=(y^{(k)})_{k=1}^K$ denote the one-hot encoded vector of label $y$, i.e., $y^{(k)}=1$ if and only if $y=k$, otherwise $y^{(k)}=0$. Then, given a training dataset $\mathcal{D}=\{\mbf{x}_i, y_i\}^{n}_{i=1}$ that consists of $n$ data points sampled \emph{i.i.d.} from the source distribution $P_S(\mbf{x}, y)$ defined over $\mathcal{X} \times \mathcal{Y}$, $f_{\bm{\theta}}$ is trained following the empirical cross-entropy loss minimization:
\begin{equation}
\label{cross_entropy_loss}
    \mathcal{L}_{\mathcal{D}}(f_{\bm{\theta}})= -\frac{1}{n}\sum_{i=1}^n\sum_{k=1}^K y^{(k)}_i\log \mathrm{s}_{\bm{\omega}}^{(k)}(f_{\mbf{g}}(\mbf{x}_i)),
\end{equation}
where
% $e^{y^{(i)}}$ denotes a one-hot vector of which only the $y^{(i)}$-th element is 1, and $
%$\mathrm{s}_k$ 
$\mathrm{s}_{\bm{\omega}}^{(k)}$ denotes the output of the softmax for the class $k$ approximating the posterior probability $P(Y\!=\!k|\mbf{x})$, i.e., $\mathrm{s}_{\bm{\omega}}^{(k)}(f_{\mbf{g}}(\mbf{x}))=\exp\{\T{\mathbf{w}_k}f_{\mbf{g}}(\mbf{x})\}/\left(\sum_{\tilde{k}} \exp\{\T{\mathbf{w}_{\tilde{k}}}f_{\mbf{g}}(\mbf{x})\}\right)$.

\paragraph{Unsupervised accuracy estimation.}
We now assume to have access to $m$ test samples from the target distribution $\mathcal{D}_{\text{test}}=\{\Tilde{\mbf{x}_i}\}_{i=1}^m \sim P_T(\mbf{x})$, where $P_T(\mbf{x}, y)\neq P_S(\mbf{x}, y)$.
% In principle, we assume \vasilii{that covariate shift holds, i.e., $P_T(\mbf{x})\!\neq\!P_S(\mbf{x}),\ P_T(y|\mbf{x})\!=\!P_S(y|\mbf{x})$}. 
For each test sample $\Tilde{\mbf{x}}_i$, we predict the label by $\Tilde{y}'_i=\argmax_{k\in\mathcal{Y}} f_{\bm{\theta}}(\Tilde{\mbf{x}}_i)$. We now want to assess the performance of $f_{\bm{\theta}}$ on a target distribution without using corresponding ground-truth labels $\{\Tilde{y}_i\}_{i=1}^m$ by estimating as accurately as possible the following quantity:
\begin{equation} 
\label{true_error}
    \operatorname{Acc}(\mathcal{D}_{\text{test}})=\frac{1}{m} \sum_{i=1}^{m}\mathds{1}(\Tilde{y}^{\prime}_i = \widetilde{y}_i),
\end{equation}
where $\mathds{1}(\cdot)$ denotes the indicator function. In practice, unsupervised accuracy estimation methods provide a proxy score $S(\mathcal{D}_{\text{test}})$ that should exhibit a linear correlation with $\operatorname{Acc}(\mathcal{D}_{\text{test}})$. The performance of such methods is measured using the coefficient of determination $R_2$ and the Spearman correlation coefficient $\rho$. 

%\section{\textsc{GdScore}: on the strong correlation between gradient norm and test accuracy}
\section{On the strong correlation between gradient norm and test accuracy}
\label{section:contributions}
We start by deriving an analytical expression of the gradient obtained when fine-tuning a source pre-trained model on new test data. We further use the intuition derived from it to propose our test accuracy estimation score and justify its effectiveness through a more thorough theoretical analysis.

% , which indicates that the final OOD error can be predicted by the one-gradient-step norm. Then we

\paragraph{A motivational example.} Below, we follow the setup considered by \citet{pmlr-v97-denevi19a,pmlr-v97-balcan19a,maml_arnoal} to develop our intuition behind the importance of gradient norm in unsupervised accuracy estimation. Our main departure point for this analysis is to consider fine-tuning: a popular approach to adapting a pre-trained model to different labeled datasets is to update either all or just a fraction of its parameters using gradient descent on the new data. To this end, let us consider the following linear regression example, where the test data from unseen domains are distributed as $X\!\sim\!\mathcal{N}(0,\sigma^x_t)$, $(Y|X\!=\!x)\!\sim\!\mathcal{N}(\theta_t x, 1)$ parameterized by the optimal regressor $\theta_t\in\R$, while the data on which the model was trained is distributed as $X\!\sim\!\mathcal{N}(0,\sigma^x_s)$, $(Y|X\!=\!x)\!\sim\!\mathcal{N}(\theta_s x, 1)$ with $\theta_s\in\R$. Consider the least-square loss over the test distribution:
% two data distributions, $x_t\!\sim\!\mathcal{N}(0,\sigma^x_t)$ and $x_s\!\sim\!\mathcal{N}(0,\sigma_s^x)$ denoting respectively the potential OOD data and the data on which the model was trained. We further consider the output variables $y_t^{\theta_t}|x_t\!\sim\!\mathcal{N}(\theta_t x_t,1)$, $y_s^{\theta_s}|x_s\!\sim\!\mathcal{N}(\theta_s x_s,1)$ with $\theta_t\!\sim\!\mathcal{N}(0,\sigma^\theta_t)$ and $\theta_s \!\sim\!\mathcal{N}(0,\sigma^\theta_s)$ denoting source and target parameters. We consider a regression task with the least-square loss:
$$
\mathcal{L}_T(c) = \frac{1}{2} \mathbb{E}_{P_T(x,y)} (y-cx)^2.
$$ 
When we do not observe the target labels, one possible solution would be to analyze fine-tuning when using the source generator $(Y|X\!=\!x)\!\sim\!\mathcal{N}(\theta_s x, 1)$ for pseudo-labeling. Then, we obtain that
% \begin{align*}
%     \quad \frac{1}{2} \nabla_c \mathbb{E}_{P_T(x)}&\mathbb{E}_{P_S(y|x)} [(y-cx)^2] = \\
%     &= \mathbb{E}_{P_T(x)}\mathbb{E}_{P_S(y|x)} [(y-cx)(-x)] \\
%     &= \mathbb{E}_{P_T(x)}\mathbb{E}_{P_S(y|x)} [cx^2 - xy] \\
%     &= (c - \theta_s) \sigma_t^x \\
%     & = ((c - \theta_t) + (\theta_t - \theta_s))\sigma_t^x.
% \end{align*}
\begin{align*}
    \frac{1}{2} \nabla_c \mathbb{E}_{P_T(x)}\mathbb{E}_{P_S(y|x)} [(y-cx)^2] 
    % &= \mathbb{E}_{p(x_t,y_t^{\theta_s}|\theta_t)} [(y-cx)(-cx)'_c] \\
    &= \mathbb{E}_{P_T(x)}\mathbb{E}_{P_S(y|x)} [(y-cx)(-x)] \\
    &= \mathbb{E}_{P_T(x)}\mathbb{E}_{P_S(y|x)} [cx^2 - xy] \\
    &= (c - \theta_s) \sigma_t^x \\
    & = ((c - \theta_t) + (\theta_t - \theta_s))\sigma_t^x.
\end{align*}
This derivation, albeit simplistic, suggests that the gradient over the target data correlates -- modulo the variance of $x$ -- with $(c - \theta_t)$, capturing how far we are from the optimal parameters of the target model, and $(\theta_s - \theta_t)$, that can be seen as a measure of dissimilarity between the distributions of the optimal source and target parameters. Intuitively, both these terms are important for predicting the test accuracy performance suggesting that the gradient itself can be a good proxy for the latter.

\subsection{Proposed approach: \textsc{GdScore}}
% In this section, we formally introduce our proposed method, GrdNorm Score, to estimate OOD error during evaluation. We start by recalling the backpropagation process of the pre-trained neural network $f_{\theta}$ from a cross-entropy loss, and then describe how to leverage the gradient norm for OOD error estimation. The detailed calculation can be found in Appendix~\ref{alg: grdnorm}.
We now formally introduce our proposed score, termed \textsc{GdScore}, to estimate test accuracy in an unsupervised manner during evaluation. We start by recalling the backpropagation process of the pre-trained neural network $f_{\theta}$ from a cross-entropy loss and then describe how to leverage the gradient norm for the unsupervised accuracy estimation. The detailed algorithm can be found in Appendix~\ref{app:algorithm}.

\paragraph{Feedforward.} Similar to the feedforward in the pre-training process shown in Eq.~\ref{feedforward}, for any given test individual $\Tilde{\mbf{x}}_i$, we have:
% Without loss of generality, a neural network $f_{\bm{\theta}}$ can be viewed as the combination of a complex feature extractor $f_{\mbf{g}}$ and a linear classification layer $f_{\bm{\omega}}$, where $g$ and $\omega$ denote their corresponding parameters. Given a test individual $\mbf{x}$, the feedforward process can be expressed as:
\begin{equation}
\label{feedforward2}
    f_{\bm{\theta}}(\Tilde{\mbf{x}}_i) = f_{\bm{\omega}}(f_{\mbf{g}}(\Tilde{\mbf{x}}_i)).
\end{equation}
As explained above, we do not observe the true labels of test data. We now detail our strategy for pseudo-labeling that allows us to obtain accurate and balanced proxies for test data labels based on accurate and potentially inaccurate model predictions.

% \textbf{Label generation strategy.} To generate labels of the test samples for later backpropagation, a natural idea is to use the pseudo labels predicted by the pre-trained model. However, this method exists an obvious drawback: we treat all the assigned labels as correct predictions when formulating the loss, ignoring underlying information from incorrect pseudo labels, leading to imprecise calculation about the magnitude of gradients. The detailed empirical evidence is shown in Section \ref{section_label_generaration}.

\paragraph{Label generation strategy.} Unconditionally generating pseudo-labels for test data under distribution shift exhibits an obvious drawback: we treat all the assigned pseudo-labels as correct predictions when calculating the loss, ignoring the fact that some examples are possibly mislabeled. Therefore, we propose the following confidence-based label-generation policy  that outputs for every  $\Tilde{\mbf{x}}_i\in\mathcal{D}_{\text{test}}$:
% ignoring the information from incorrect pseudo-labels that we would have observed in practice. Therefore, instead of leveraging the whole pseudo-label set, given input $\Tilde{\mbf{x}}_i$ our label-generation process can be expressed as follows:
% Therefore, instead of leveraging the whole pseudo-label set, we assign proper labels to each test sample based on its confidence. In details, for a given input $\Tilde{\mbf{x}}_i$, when its confidence is larger than a threshold, we assign its label as the pseudo label. Otherwise, we randomly sample a label from the label space with equal probability. Formally, the label generation process can be expressed as follows:
\begin{equation}
\label{label_generation}
\Tilde{y}'_i=
\begin{cases}
\argmax_k f_{\bm{\theta}}(\Tilde{\mbf{x}}_i), & \ \max_k 
\mathrm{s}_{\bm{\omega}}^{(k)}(f_{\mbf{g}}(\Tilde{\mbf{x}}_i)) > \tau\\
\Tilde{y}' \sim U[1,K], & \  \text{otherwise} %otherwise,
\end{cases}
\end{equation}
where $\tau$ denotes the threshold value, and $U[1,k]$ denotes the discrete uniform distribution with outcomes $\{1,\dots,K\}$. In a nutshell, we assign the predicted label to $\Tilde{\mbf{x}}_i$, when the prediction confidence is larger than a threshold while using a randomly sampled label from the label space otherwise. The detailed empirical evidence justifying this choice is shown in Section \ref{section_ablation}, and we discuss the choice of proper threshold $\tau$ in Appendix~\ref{app:choice_threshold}. From the theoretical point of view, our approach assumes that the classifier makes mistakes mostly on data with low prediction confidence, for which we deliberately assign noisy pseudo-labels. \citet{feofanov2019transductive} used a similar approach to derive an upper bound on the test error and proved its tightness in the case where the assumption is satisfied. We discuss this matter in more detail in Appendix~\ref{app:calibration}.
% where $t$ denotes the threshold value, $\mathrm{s}$ denotes the Softmax function, and $U[1, K]$ denotes the discrete uniform distribution with outcomes $1,\dots,K$
% To obtain the precise magnitude of gradients needed to adjust the pre-trained model for fitting the OOD distribution, we 

% \textbf{Backpropagation.} Our method requires neither gradients of the whole parameters nor iterative training. We only need to calculate the gradients w.r.t. the weights of the classification layer with one epoch by backpropagating the standard cross-entropy loss. The loss can be obtained by:
\paragraph{Backpropagation.} To estimate our score, we calculate the gradients w.r.t. the weights of the classification layer $\bm{\omega}$ during the first epoch backpropagated over the standard cross-entropy loss defined by:
% i.e., one-step gradients of the classification layer,
\begin{equation}
    % \mathcal{L}(f_{\bm{\theta}}(\Tilde{\mbf{x}}), \Tilde{y}')=\mathbb{E}_{\mathcal{\mathcal{\Tilde{D}}}}}(\ell(f(\Tilde{x};\theta), \Tilde{y}')) \approx -\frac{1}{m}\sum_{i=1}^me^{\Tilde{y}_i}\log \mathrm{s}(f(f(\Tilde{x}_i; g);\omega)).
 \mathcal{L}_{\mathcal{D_{\text{test}}}}(f_{\bm{\theta}})= -\frac{1}{m}\sum_{i=1}^m\sum_{k=1}^K \Tilde{y}'^{(k)}_i\log \mathrm{s}_{\bm{\omega}}^{(k)}(f_{\mbf{g}}(\mbf{\Tilde{x}}_i)),
    \label{eq:cross_entropy}
\end{equation}
where each unlabeled instance $\mbf{\Tilde{x}}_i$ is pseudo-labeled following Eq.~\ref{label_generation}. Then, the gradient of the classification layer $\bm{\omega}$ is evaluated as follows:
\begin{equation}
    %\frac{\partial \mathcal{L}_{\mathcal{D_{\text{test}}}}(f_{\bm{\theta}})}{\partial \bm{\omega}} = -\frac{1}{m}\sum_{i=1}^m\sum_{k=1}^K \frac{\partial \Tilde{y}^{'(k)}_i\log \mathrm{s}_{\bm{\omega}}^{(k)}(f_{\mbf{g}}(\mbf{\Tilde{x}}_i))}{\partial \bm{\omega}}.
    \nabla_{\bm{\omega}} \mathcal{L}_{\mathcal{D}_{\text{test}}}(f_{\bm{\theta}}) = -\frac{1}{m}\sum_{i=1}^m\sum_{k=1}^K \nabla_{\bm{\omega}} \left( \Tilde{y}'^{(k)}_i\log \mathrm{s}_{\bm{\omega}}^{(k)}(f_{\mbf{g}}(\mbf{\Tilde{x}}_i)) \right).
    \label{eq:gradients}
\end{equation}
Note that our method requires neither gradients of the whole parameter set of the pre-trained model nor iterative training. This makes it highly computationally efficient.

\paragraph{\textsc{GdScore}.} Now, we can define \textsc{GdScore} using a vector norm of gradients of the last layer. The score is expressed as follows: 
\begin{equation}
    S(\mathcal{D}_{\text{test}}) = %\bigg|\bigg|\frac{\partial \mathcal{L}_{\mathcal{D_{\text{test}}}}(f_{\bm{\theta}})}{\partial \bm{\omega}}\bigg|\bigg|_p,
    \lVert \nabla_{\bm{\omega}} \mathcal{L}_{\mathcal{D}_{\text{test}}}(f_{\bm{\theta}}) \rVert_p,
    \label{eq:score}
\end{equation}
where $||\cdot||_p$ denotes $L_p$-norm.

\subsection{Theoretical analysis}
In this section, we provide theoretical insights into our method. We first clarify the connection between the true target cross-entropy error and the norm of the gradients. Then, we show that the gradient norm is upper-bounded by a weighted sum of the norm of the inputs. 

\paragraph{Notations.} For the sake of simplicity, we assume the feature extractor $f_{\mbf{g}}$ is fixed and, by abuse of notation, we use $\mbf{x}$ instead of $f_{\mbf{g}}(\mbf{x})$. 
%further omit the feature extractor $f_{\mathbf{g}}$ and consider only the linear classification layer $f_{\bm{\omega}}$. 
Reusing the notations introduced in Section~\ref{section:background}, the true target cross-entropy error writes
\begin{align*}
 \mathcal{L}_T(\bm{\omega})= -\E_{P_T(\mbf{x},y)}\sum_k y^{(k)}\log \mathrm{s}_{\bm{\omega}}^{(k)}(\mbf{x}),
\end{align*}
where $\bm{\omega} = (\mbf{w}_k)_{k=1}^K \in \mathbb{R}^{D \times K}$ are the parameters of the linear classification layer $f_{\bm{\omega}}$. For the ease of notation, the gradient of $\mathcal{L}_T$ w.r.t $\bm{\omega}$ is denoted by $\nabla \mathcal{L}_T$.

The following theorem, whose proof we defer to Appendix~\ref{app:true-risk-grad-norm}, makes the connection between the true risk and the $L_p$-norm of the gradient explicit. 
\begin{thm}[Connection between the true risk and the $L_p$-norm of the gradient]
\label{thm:target-risk-grad-norm}
Let $\mbf{c}\in\R^{D\times K}$ and $\mbf{c}'\in\R^{D\times K}$ be two linear classifiers. For any $p, q \geq 1$ such that $\frac{1}{p} + \frac{1}{q} = 1$, we have that
    \begin{align*}
    |\mathcal{L}_T(\mbf{c}') - \mathcal{L}_T(\mbf{c})| \leq \max_{\bm{\omega}\in\{\mbf{c}',\mbf{c}\}}(\norm{\nabla \mathcal{L}_T(\bm{\omega})}_p)\cdot \norm{\mbf{c}'-\mbf{c}}_q.
    \end{align*}
\end{thm}
The left-hand side here is the difference in terms of the true risks obtained for the same distribution with two different classifiers. The right-hand side shows how this difference is controlled by the maximum gradient norm over the two classifiers and a term capturing how far the two are apart. 

In the context of the proposed approach, we want to know the true risk of the source classifier $\bm{\omega}_s$ on the test data and its change after one step of gradient descent. The following corollary applies Theorem \ref{thm:target-risk-grad-norm} to characterize this exact case. The proof is deferred to Appendix~\ref{app:after-one-grad-update}.
\begin{cor}[Connection after one gradient update]
\label{cor:after-one-grad-update}
    Let $\mbf{c}$ be the classifier obtained from $\bm{\omega}_s$ after one gradient descent step, i.e., $\mbf{c}=\bm{\omega}_s-\eta\cdot \nabla \mathcal{L}_T(\bm{\omega}_s)$ with $\eta\geq 0$. Then, when $\bm{\omega}\in\{\bm{\omega}_s, \mbf{c}\}$, we have that
    \begin{align*}
        |\mathcal{L}_T(\bm{\omega}_s)\!-\!\mathcal{L}_T(\mbf{c})| 
        &\!\leq\!\eta\max_{\bm{\omega}}(\norm{\nabla\mathcal{L}_T(\bm{\omega})}_p)\norm{\nabla\mathcal{L}_T(\bm{\omega}_s)}_q.
    \end{align*}
\end{cor}

Note that in this case $\mathcal{L}_T(\bm{\omega}_s)$ can be seen as a term providing the true test risk after pseudo-labeling it with the source classifier. This shows the importance of pseudo-labeling as it acts as a departure point for obtaining a meaningful estimate of the right-hand side. When the latter is meaningful, the gradient norm on the right-hand side controls it together with a magnitude that tells us how far we went after one step of backpropagation. 

In the next theorem, we provide an upper bound on the $L_p$-norm of the gradient as a weighted sum of the $L_p$-norm of the inputs. The proof is deferred to Appendix~\ref{app:upper_bound_norm_grad}.
\begin{thm}[Upper-bounding the norm of the gradient] 
\label{thm:upper_bound_norm_grad}
For any $p\geq1$ and $\forall\,\bm{\omega}\in\mathbb{R}^{D \times K}$, the $L_p$-norm of the gradient can be upper-bounded as follows:
\begin{align*}
    % \norm{\nabla\mathcal{L}_{\mbf{x},y}}_p &= \left( \alpha(\bm{\omega},\mbf{x},y)\cdot \norm{\mbf{x}}_p^p \right)^{1/p},\\
    % \text{if }p\geq 1, \text{ then }
    \norm{\nabla\mathcal{L}_T(\bm{\omega})}_p &\leq \E_{P_T(\mbf{x},y)} \alpha(\bm{\omega},\mbf{x},y)\cdot \norm{\mbf{x}}_p, 
\end{align*}
where $\alpha(\bm{\omega},\mbf{x}, y)\!=\!1\!-\!\mathrm{s}_{\bm{\omega}}^{(k_y)}(\mbf{x})$, with $k_y$ such that $y^{(k_y)}\!=\!1$.
\end{thm}
% To complete the theoretical analysis, we adapt the previous result to the proposed \textsc{GrdNorm} Score (proof deferred to Appendix~\ref{app:upper_bound_grdnorm}).

% \begin{cor}[Upper-bounding the \textsc{GrdNorm}]
% \label{cor:upper_bound_GRDNORM}
% For any $p \geq 1$, $S(\mathcal{D}_{\text{test}})$ can be upper-bounded as follows:
% \begin{align*}
%     S(\mathcal{D}_{\text{test}}) \leq \left( \frac{1}{m}\sum_{i=1}^m \alpha(\bm{\omega},\tilde{\mbf{x}}_i,\tilde{y}_i)\cdot \norm{\tilde{\mbf{x}}_i}_p^p \right)^{1/p}.
% \end{align*}
% \end{cor}
%Note that $\alpha(\bm{\omega},\mbf{x}, y)=\left(1-\mathrm{s}_{\bm{\omega}}^{(k_y)}(\mbf{x})\right)^p$, with $k_y$ such that $y^{(k_y)}\!=\!1$. 
Hence, the norm of the gradient is upper-bounded by a weighted combination of the norm of the inputs, where the weight $\alpha(\bm{\omega},\mbf{x}, y) \in [0,1]$ conveys how well the model predicts on $\mbf{x}$. In the case of perfect classification, the upper bound is tight and equals $0$. 
%\begin{cor}[Case when $0 < p < 1$]
    %Let $\mbf{c}$ be the classifier obtained from $\bm{\omega}_s$ after one gradient descent step, i.e., $\mbf{c}=\bm{\omega}_s-\eta\cdot \nabla %\mathcal{L}_T(\bm{\omega}_s)$ with $\eta\in[0,1]$. We have that
    %\begin{equation*}
        %\lVert \nabla \mathcal{L}_T(\bm{\omega}_s) \rVert_p \leq \frac{1}{\eta} \lvert \lVert \mbf{c} \rVert_p - \lVert \bm{\omega}_s \rVert _p \rvert.
    %\end{equation*}
%\end{cor}
%\begin{proof}
    %Using the reverse Minkowski inequality, we have that
    %\begin{align*}
        %& \lVert \bm{\omega}_s \rVert _p = \lVert \mbf{c} + \eta\cdot \nabla \mathcal{L}_T(\bm{\omega}_s) \rVert _p \geq \lVert \mbf{c} \rVert_p + \eta \cdot %\lVert \nabla \mathcal{L}_T(\bm{\omega}_s) \rVert_p \\
        %\implies & \lVert \bm{\omega}_s \rVert _p - \lVert \mbf{c} \rVert_p \geq \eta \cdot \lVert \nabla \mathcal{L}_T(\bm{\omega}_s) \rVert_p.
    %\end{align*}
%In the same fashion, we have that
    %\begin{align*}
        %& \lVert \mbf{c} \rVert _p = \lVert \bm{\omega}_s - \eta\cdot \nabla \mathcal{L}_T(\bm{\omega}_s) \rVert _p \geq \lVert \bm{\omega}_s \rVert_p + \eta %\cdot \lVert \nabla \mathcal{L}_T(\bm{\omega}_s) \rVert_p \\
        %\implies & \lVert \mbf{c} \rVert_p - \lVert \bm{\omega}_s \rVert _p \geq \eta \cdot \lVert \nabla \mathcal{L}_T(\bm{\omega}_s) \rVert_p.
    %\end{align*}
%We obtain the desired upper bound by combining those results.
%\end{proof}
% \paragraph{Link with \textsc{GrdNorm}.} 
In practice, as we do not have access to the true risk, the gradients can be approximated by the proposed \textsc{GdScore} that requires to pseudo-label test data by Eq.~\ref{label_generation}. As we said earlier, this implies that the model has to be well calibrated (see Appendix~\ref{app:calibration}), which is conventional to assume for self-training methods \citep{amini2022self} including the approach of \citet{yu2022predicting}. Then, the network projects test data into the low confidence regions, and the gradient for these examples will be large as we need to update $\bm{\omega}_s$ significantly to fit them.

\begin{figure*}[!t]
    % \centering
    \includegraphics[width=\linewidth]{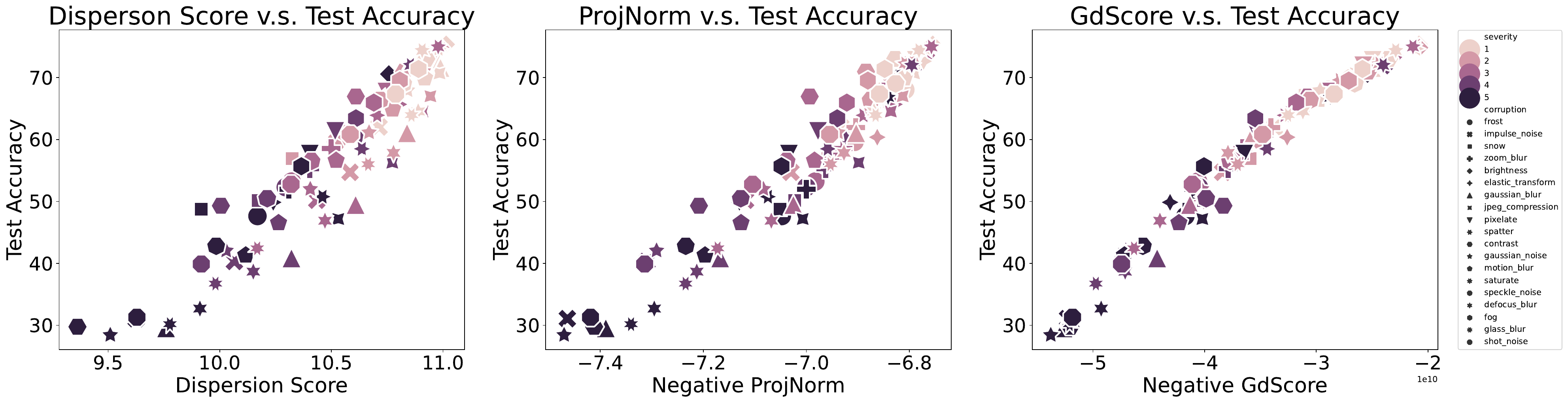}
    \vspace{-10pt}
    \caption{Test accuracy prediction versus True test accuracy on Entity-13 with ResNet18. We compare the performance of \textsc{GdScore} with that of Dispersion Score and ProjNorm via scatter plots. Each point represents one dataset under certain corruption and certain severity, where different shapes represent different types of corruption, and darker color represents the higher severity level.}
    \label{fig:scatters}
    \vspace{-15pt}
\end{figure*}

\section{Experiments}
\label{section:experiments}

\begin{wrapfigure}{r}{0.5\textwidth}
    \centering
    \vspace{-5pt}
    \includegraphics[width=0.75\linewidth]{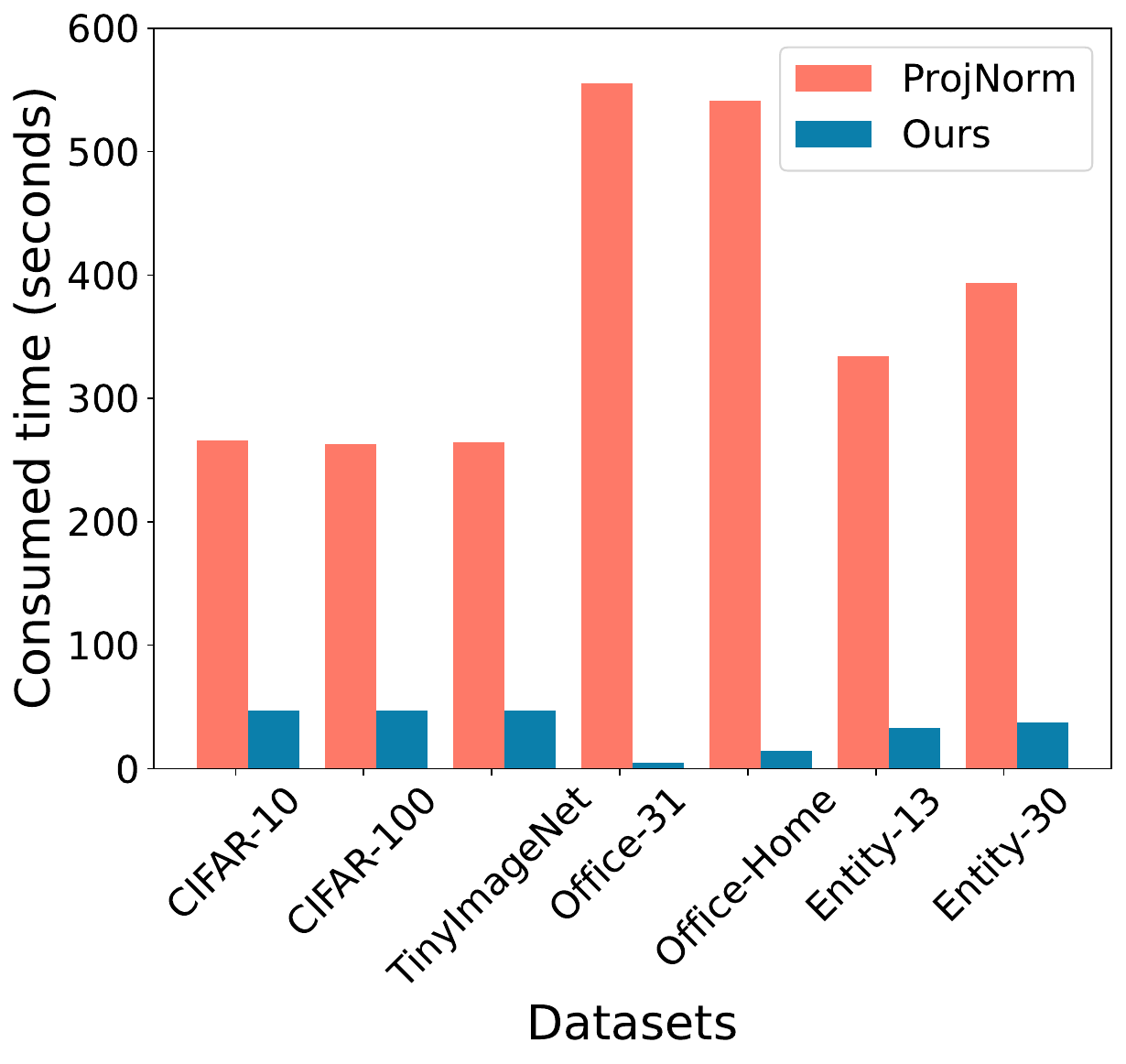}
    \caption{Runtime comparison of two self-training approaches with ResNet50.}
    \label{fig:time}
    \vspace{-20pt}
\end{wrapfigure}
% In this section, we verify the effectiveness of our method on 11 benchmark datasets representing different kinds of distribution n thshift with 3 types of model structures compared with 8 baselines.
%In this section, we verify the effectiveness of our method in terms of OOD error estimation and computational efficiency. 
\subsection{Experimental setup}
\paragraph{Pre-training datasets.} For pre-training the neural network, we use CIFAR-10, CIFAR-100 \citep{krizhevsky2009learning}, TinyImageNet \citep{le2015tiny}, ImageNet \citep{deng2009imagenet}, Office-31 \citep{saenko2010adapting}, Office-Home \citep{venkateswara2017deep}, Camelyon17-WILDS \citep{koh2021wilds}, and BREEDS \citep{santurkar2020breeds} which leverages class hierarchy of ImageNet \citep{deng2009imagenet} to create 4 datasets including Living-17, Nonliving-26, Entity-13 and Entity-30. 
In particular, to avoid time-consuming training, we directly utilize publicly available models pre-trained on Imagenet. For Office-31 and Office-Home, we train a neural network on every domain.
% Office-31 \citep{saenko2010adapting} is collected from 3 different domains: Amazon, DSLR, and Webcam, while Office-Home \citep{venkateswara2017deep} consists of 4 domains: Artistic, Clip Art, Product and Real-World. For these two datasets, we train a neural network on every domain, which means we have $3$ pre-trained models for Office-31 and $4$ pre-trained models for Office-Home.

\begin{table*}[!t]
    \centering
    \caption{Performance comparison on 11 benchmark datasets with ResNet18, ResNet50, and WRN-50-2, where $R^2$ refers to coefficients of determination, and $\rho$ refers to the absolute value of Spearman correlation coefficients (higher is better). The best results are highlighted in \textbf{bold}.}
    \renewcommand\arraystretch{1.1}
    \resizebox{0.95\textwidth}{!}{
    \setlength{\tabcolsep}{1mm}{
    \begin{tabular}{cccccccccccccccccccc}
        \toprule
        \multirow{2}{*}{Dataset} &\multirow{2}{*}{Network} &\multicolumn{2}{c}{Rotation} &\multicolumn{2}{c}{ConfScore} &\multicolumn{2}{c}{Entropy} &\multicolumn{2}{c}{AgreeScore} &\multicolumn{2}{c}{ATC} &\multicolumn{2}{c}{Fr\'{e}chet} &\multicolumn{2}{c}{Dispersion}&\multicolumn{2}{c}{ProjNorm} &\multicolumn{2}{c}{Ours}\\
        \cline{3-20}
        & &$R^2$ &$\rho$ &$R^2$ &$\rho$&$R^2$ &$\rho$&$R^2$ &$\rho$&$R^2$ &$\rho$&$R^2$ &$\rho$&$R^2$ &$\rho$ &$R^2$ &$\rho$&$R^2$ &$\rho$\\
        \midrule
         \multirow{4}{*}{CIFAR 10} & ResNet18 &0.822 &0.951 &0.869 &0.985 &0.899 &0.987 &0.663 &0.929 &0.884 &0.985 &0.950 &0.971 &0.968 &0.990 &0.936 &0.982 &\textbf{0.971} &\textbf{0.994} \\
          & ResNet50 &0.835 &0.961 &0.935 &0.993 &0.945 &\textbf{0.994} &0.835 &0.985 &0.946 &\textbf{0.994} &0.858 &0.964 &\textbf{0.987} &0.990 &0.944 &0.989 &0.969 &0.993\\
          & WRN-50-2 &0.862 &0.976 &0.943 &
          \textbf{0.994} &0.942 &\textbf{0.994} &0.856 &0.986 &0.947 &\textbf{0.994} &0.814 &0.973 &0.962 &0.988 &0.961 &0.989 &\textbf{0.971} &\textbf{0.994}\\
          \cline{2-20}
          & \textcolor[rgb]{0.0, 0.53, 0.74}{Average} &\textcolor[rgb]{0.0, 0.53, 0.74}{0.840} &\textcolor[rgb]{0.0, 0.53, 0.74}{0.963} &\textcolor[rgb]{0.0, 0.53, 0.74}{0.916} &\textcolor[rgb]{0.0, 0.53, 0.74}{0.991} &\textcolor[rgb]{0.0, 0.53, 0.74}{0.930} &\textcolor[rgb]{0.0, 0.53, 0.74}{0.992} &\textcolor[rgb]{0.0, 0.53, 0.74}{0.785} &\textcolor[rgb]{0.0, 0.53, 0.74}{0.967} &\textcolor[rgb]{0.0, 0.53, 0.74}{0.926} &\textcolor[rgb]{0.0, 0.53, 0.74}{0.991} &\textcolor[rgb]{0.0, 0.53, 0.74}{0.874} &\textcolor[rgb]{0.0, 0.53, 0.74}{0.970} &\textcolor[rgb]{0.0, 0.53, 0.74}{\textbf{0.972}} &\textcolor[rgb]{0.0, 0.53, 0.74}{0.990} &\textcolor[rgb]{0.0, 0.53, 0.74}{0.947} &\textcolor[rgb]{0.0, 0.53, 0.74}{0.987}  &\textcolor[rgb]{0.0, 0.53, 0.74}{\textbf{0.970}} &\textcolor[rgb]{0.0, 0.53, 0.74}{\textbf{0.994}}\\
          \midrule
    \multirow{4}{*}{CIFAR 100} & ResNet18 &0.860 &0.936 &0.916 &0.985 &0.891 &0.979 &0.902 &0.973 &0.938 &0.986 &0.888 &0.968 &0.952 &0.988 &0.979 &0.980 &\textbf{0.987} &\textbf{0.996}\\
          & ResNet50 &0.908 &0.962 &0.919 &0.984 &0.884 &0.977 &0.922 &0.982 &0.921 &0.984 &0.837 &0.972 &0.951 &0.985 &0.988 &0.991 &\textbf{0.991} &\textbf{0.997}\\
          & WRN-50-2 &0.924 &0.970 &0.971 &0.984 &0.968 &0.981 &0.955 &0.977 &0.978 &0.993 &0.865 &0.987 &0.980 &0.991 &0.990 &0.991 &\textbf{0.995} &\textbf{0.998}\\
          \cline{2-20}
          & \textcolor[rgb]{0.0, 0.53, 0.74}{Average} &\textcolor[rgb]{0.0, 0.53, 0.74}{0.898} &\textcolor[rgb]{0.0, 0.53, 0.74}{0.956} &\textcolor[rgb]{0.0, 0.53, 0.74}{0.936} &\textcolor[rgb]{0.0, 0.53, 0.74}{0.987} &\textcolor[rgb]{0.0, 0.53, 0.74}{0.915} &\textcolor[rgb]{0.0, 0.53, 0.74}{0.983} &\textcolor[rgb]{0.0, 0.53, 0.74}{0.927} &\textcolor[rgb]{0.0, 0.53, 0.74}{0.982} &\textcolor[rgb]{0.0, 0.53, 0.74}{0.946} &\textcolor[rgb]{0.0, 0.53, 0.74}{0.988} &\textcolor[rgb]{0.0, 0.53, 0.74}{0.864} &\textcolor[rgb]{0.0, 0.53, 0.74}{0.976} &\textcolor[rgb]{0.0, 0.53, 0.74}{0.962} &\textcolor[rgb]{0.0, 0.53, 0.74}{0.988} &\textcolor[rgb]{0.0, 0.53, 0.74}{0.985} &\textcolor[rgb]{0.0, 0.53, 0.74}{0.987} &\textcolor[rgb]{0.0, 0.53, 0.74}{\textbf{0.991}} &\textcolor[rgb]{0.0, 0.53, 0.74}{\textbf{0.997}}\\
          \midrule
   \multirow{4}{*}{TinyImageNet} & ResNet18 &0.786 &0.946 &0.670 &0.869 &0.592 &0.842 &0.561 &0.853 &0.751 &0.945 &0.826 &0.970 &0.966 &0.986 &0.970 &0.981 &\textbf{0.971} &\textbf{0.994}\\
          & ResNet50 &0.786 &0.947 &0.670 &0.869 &0.651 &0.892 &0.560 &0.853 &0.751 &0.945 &0.826 &0.971 &0.977 &0.986 &0.979 &0.987 &\textbf{0.980} &\textbf{0.995}\\
          & WRNt-50-2 &0.878 &0.967 &0.757 &0.951 &0.704 &0.935 &0.654 &0.904 &0.635 &0.897 &0.884 &0.984 &0.968 &0.986 &0.965 &0.983 &\textbf{0.975} &\textbf{0.996}\\
          \cline{2-20}
          & \textcolor[rgb]{0.0, 0.53, 0.74}{Average} &\textcolor[rgb]{0.0, 0.53, 0.74}{0.805} &\textcolor[rgb]{0.0, 0.53, 0.74}{0.959} &\textcolor[rgb]{0.0, 0.53, 0.74}{0.727} &\textcolor[rgb]{0.0, 0.53, 0.74}{0.920} &\textcolor[rgb]{0.0, 0.53, 0.74}{0.650} &\textcolor[rgb]{0.0, 0.53, 0.74}{0.890} &\textcolor[rgb]{0.0, 0.53, 0.74}{0.599} &\textcolor[rgb]{0.0, 0.53, 0.74}{0.878} &\textcolor[rgb]{0.0, 0.53, 0.74}{0.693} &\textcolor[rgb]{0.0, 0.53, 0.74}{0.921} &\textcolor[rgb]{0.0, 0.53, 0.74}{0.847} &\textcolor[rgb]{0.0, 0.53, 0.74}{0.976} &\textcolor[rgb]{0.0, 0.53, 0.74}{0.970} &\textcolor[rgb]{0.0, 0.53, 0.74}{0.987} &\textcolor[rgb]{0.0, 0.53, 0.74}{0.972} &\textcolor[rgb]{0.0, 0.53, 0.74}{0.984} &\textcolor[rgb]{0.0, 0.53, 0.74}{\textbf{0.976}} &\textcolor[rgb]{0.0, 0.53, 0.74}{\textbf{0.995}}\\
          \midrule
    \multirow{4}{*}{ImageNet} & ResNet18 &- &- &0.979 &0.991 &0.963 &0.991 &- &- &0.974 &0.983 &0.802 &0.974  &0.940 &0.971 &0.975 &0.993 &\textbf{0.986} &\textbf{0.996}\\
          & ResNet50 &- &- &0.980 &0.994 &0.967 &0.992 &- &- &0.970 &0.983 &0.855 &0.974 &0.938 &0.968 &0.986 &0.993 &\textbf{0.987} &\textbf{0.996}\\
          & WRNt-50-2 &- &- &0.983 &0.991 &0.963 &0.991 &- &- &0.983 &0.993 &0.909 &0.988 &0.939 &0.976 &0.978 &0.993 &\textbf{0.984} &\textbf{0.998}\\
          \cline{2-20}
          & \textcolor[rgb]{0.0, 0.53, 0.74}{Average} &\textcolor[rgb]{0.0, 0.53, 0.74}{-} &\textcolor[rgb]{0.0, 0.53, 0.74}{-} &\textcolor[rgb]{0.0, 0.53, 0.74}{0.981} &\textcolor[rgb]{0.0, 0.53, 0.74}{0.993} &\textcolor[rgb]{0.0, 0.53, 0.74}{0.969} &\textcolor[rgb]{0.0, 0.53, 0.74}{0.992} &\textcolor[rgb]{0.0, 0.53, 0.74}{-} &\textcolor[rgb]{0.0, 0.53, 0.74}{-} &\textcolor[rgb]{0.0, 0.53, 0.74}{0.976} &\textcolor[rgb]{0.0, 0.53, 0.74}{0.987} &\textcolor[rgb]{0.0, 0.53, 0.74}{0.855} &\textcolor[rgb]{0.0, 0.53, 0.74}{0.979} &\textcolor[rgb]{0.0, 0.53, 0.74}{0.939} &\textcolor[rgb]{0.0, 0.53, 0.74}{0.972} &\textcolor[rgb]{0.0, 0.53, 0.74}{0.980} &\textcolor[rgb]{0.0, 0.53, 0.74}{0.993}  &\textcolor[rgb]{0.0, 0.53, 0.74}{\textbf{0.986}} &\textcolor[rgb]{0.0, 0.53, 0.74}{\textbf{0.997}}\\   

     \midrule
         \multirow{4}{*}{Office-31} & ResNet18 &\textbf{0.753} &\textbf{0.942} &0.470 &0.828 &0.322 &0.714 &0.003 &0.085 &0.843 &0.942 &0.143 &0.257  &0.618 &0.714 &0.099 &0.428 &0.675 &0.829\\
          & ResNet50 &0.391 &0.828 &0.485 &0.828 &0.354 &0.828 &0.011 &0.463 &0.532 &0.485 &0.034 &0.257  &0.578 &0.714 &0.240 &0.428 &\textbf{0.604} &\textbf{0.829}\\
          & WRN-50-2 &0.577 &0.6 &0.524 &0.714 &0.424 &0.714 &0.002 &0.257 &0.405 &0.942 &0.034 &0.142  &\textbf{0.671} &0.714&0.147 &0.143 &0.544 &\textbf{0.829}\\
          \cline{2-20}
          & \textcolor[rgb]{0.0, 0.53, 0.74}{Average} &\textcolor[rgb]{0.0, 0.53, 0.74}{0.567} &\textcolor[rgb]{0.0, 0.53, 0.74}{0.790} &\textcolor[rgb]{0.0, 0.53, 0.74}{0.493} &\textcolor[rgb]{0.0, 0.53, 0.74}{0.790} &\textcolor[rgb]{0.0, 0.53, 0.74}{0.367} &\textcolor[rgb]{0.0, 0.53, 0.74}{0.276} &\textcolor[rgb]{0.0, 0.53, 0.74}{0.006} &\textcolor[rgb]{0.0, 0.53, 0.74}{0.211} &\textcolor[rgb]{0.0, 0.53, 0.74}{0.593} &\textcolor[rgb]{0.0, 0.53, 0.74}{0.790} &\textcolor[rgb]{0.0, 0.53, 0.74}{0.071} &\textcolor[rgb]{0.0, 0.53, 0.74}{0.047} &\textcolor[rgb]{0.0, 0.53, 0.74}{\textbf{0.622}} &\textcolor[rgb]{0.0, 0.53, 0.74}{0.714} &\textcolor[rgb]{0.0, 0.53, 0.74}{0.162} &\textcolor[rgb]{0.0, 0.53, 0.74}{0.333}  &\textcolor[rgb]{0.0, 0.53, 0.74}{0.608} &\textcolor[rgb]{0.0, 0.53, 0.74}{\textbf{0.829}}\\
          \midrule
    \multirow{4}{*}{Office-Home} & ResNet18 &0.822 &0.930 &0.795 &0.909 &0.761 &0.881 &0.054 &0.146 &0.571 &0.615 &0.605 &0.755 &0.453 &0.664  &0.064 &0.202 &\textbf{0.876} &\textbf{0.909}\\
          & ResNet50 &\textbf{0.851} &\textbf{0.944} &0.769 &0.895 &0.742 &0.853 &0.026 &0.216 &0.487 &0.734 &0.607 &0.685 &0.383 &0.727 &0.169 &0.475  &0.829 &\textbf{0.944} \\
          & WRN-50-2 &\textbf{0.823} &\textbf{0.958} &0.741 &0.874 &0.696 &0.846 &0.132 &0.405 &0.383 &0.643 &0.589 &0.706 &0.456 &0.713 &0.172 &0.531  &0.809 &0.916\\
          \cline{2-20}
          & \textcolor[rgb]{0.0, 0.53, 0.74}{Average} & \textcolor[rgb]{0.0, 0.53, 0.74}{0.832} & \textcolor[rgb]{0.0, 0.53, 0.74}{0.944} & \textcolor[rgb]{0.0, 0.53, 0.74}{0.768} & \textcolor[rgb]{0.0, 0.53, 0.74}{0.892} & \textcolor[rgb]{0.0, 0.53, 0.74}{0.733} & \textcolor[rgb]{0.0, 0.53, 0.74}{0.860} & \textcolor[rgb]{0.0, 0.53, 0.74}{0.071} & \textcolor[rgb]{0.0, 0.53, 0.74}{0.256} & \textcolor[rgb]{0.0, 0.53, 0.74}{0.480} & \textcolor[rgb]{0.0, 0.53, 0.74}{0.664} & \textcolor[rgb]{0.0, 0.53, 0.74}{0.601} & \textcolor[rgb]{0.0, 0.53, 0.74}{0.715}  & \textcolor[rgb]{0.0, 0.53, 0.74}{0.431} & \textcolor[rgb]{0.0, 0.53, 0.74}{0.702}  & \textcolor[rgb]{0.0, 0.53, 0.74}{0.135} & \textcolor[rgb]{0.0, 0.53, 0.74}{0.403}&\textcolor[rgb]{0.0, 0.53, 0.74}{\textbf{0.837}} &\textcolor[rgb]{0.0, 0.53, 0.74}{\textbf{0.923}}\\
          \midrule
          \multirow{4}{*}{Camelyon17-WILDS} & ResNet18 &0.944 &\textbf{1.000} &0.980 &\textbf{1.000} &0.980 &\textbf{1.000} &0.977 &\textbf{1.000} &0.981 &\textbf{1.000 }&0.988 &\textbf{1.000} &0.992 &\textbf{1.000} &0.612 &0.500 &\textbf{0.996} &\textbf{1.000}\\
          & ResNet50 &0.931 &\textbf{1.000} &0.994 &\textbf{1.000} &0.993 &\textbf{1.000} &0.998 &\textbf{1.000} &0.993 &\textbf{1.000} &0.971 &\textbf{1.000} &0.012 &0.500 &0.811 &\textbf{1.000}  &\textbf{0.999} &\textbf{1.000}\\
          & WRN-50-2 &0.918 &\textbf{1.000} &0.944 &\textbf{1.000} &0.945 &\textbf{1.000} &0.965 &\textbf{1.000} &0.942 &\textbf{1.000} &0.994 &\textbf{1.000} &0.001 &0.500 &0.789 &0.500 &\textbf{0.997} &\textbf{1.000}\\
          \cline{2-20}
          & \textcolor[rgb]{0.0, 0.53, 0.74}{Average} & \textcolor[rgb]{0.0, 0.53, 0.74}{0.931} & \textcolor[rgb]{0.0, 0.53, 0.74}{\textbf{1.000}} & \textcolor[rgb]{0.0, 0.53, 0.74}{0.973} & \textcolor[rgb]{0.0, 0.53, 0.74}{\textbf{1.000}} & \textcolor[rgb]{0.0, 0.53, 0.74}{0.980} & \textcolor[rgb]{0.0, 0.53, 0.74}{\textbf{1.000}} & \textcolor[rgb]{0.0, 0.53, 0.74}{0.982} & \textcolor[rgb]{0.0, 0.53, 0.74}{\textbf{1.000}} & \textcolor[rgb]{0.0, 0.53, 0.74}{0.972} & \textcolor[rgb]{0.0, 0.53, 0.74}{\textbf{1.000}}  & \textcolor[rgb]{0.0, 0.53, 0.74}{0.984} & \textcolor[rgb]{0.0, 0.53, 0.74}{\textbf{1.000}} & \textcolor[rgb]{0.0, 0.53, 0.74}{0.334} & \textcolor[rgb]{0.0, 0.53, 0.74}{0.667} & \textcolor[rgb]{0.0, 0.53, 0.74}{0.737} & \textcolor[rgb]{0.0, 0.53, 0.74}{0.667}  & \textcolor[rgb]{0.0, 0.53, 0.74}{\textbf{0.998}} & \textcolor[rgb]{0.0, 0.53, 0.74}{\textbf{1.000}}\\
          \midrule
           \multirow{4}{*}{Entity-13} & ResNet18 &0.927 &0.961 &0.795 &0.940 &0.794 &0.935 &0.543 &0.919 &0.823 &0.945 &0.950 &0.981 &0.937 &0.968 &0.952 &0.981 &\textbf{0.969} &\textbf{0.991}\\
          & ResNet50 &0.932 &0.976 &0.728 &0.941 &0.698 &0.928 &0.901 &0.964 &0.783 &0.950 &0.903 &0.959 &0.764 &0.892 &0.944 &0.974 &\textbf{0.960} &\textbf{0.995}\\
          & WRN-50-2 &0.939 &0.983 &0.930 &0.977 &0.919 &0.973 &0.871 &0.935 &0.936 &0.980 &0.906 &0.958 &0.815 &0.905 &0.950 &0.977 &\textbf{0.968} &\textbf{0.995}\\
          \cline{2-20}
          & \textcolor[rgb]{0.0, 0.53, 0.74}{Average} & \textcolor[rgb]{0.0, 0.53, 0.74}{0.933} & \textcolor[rgb]{0.0, 0.53, 0.74}{0.973} & \textcolor[rgb]{0.0, 0.53, 0.74}{0.817} & \textcolor[rgb]{0.0, 0.53, 0.74}{0.953} & \textcolor[rgb]{0.0, 0.53, 0.74}{0.804} & \textcolor[rgb]{0.0, 0.53, 0.74}{0.945} & \textcolor[rgb]{0.0, 0.53, 0.74}{772} & \textcolor[rgb]{0.0, 0.53, 0.74}{0.939} & \textcolor[rgb]{0.0, 0.53, 0.74}{0.847} & \textcolor[rgb]{0.0, 0.53, 0.74}{0.958} & \textcolor[rgb]{0.0, 0.53, 0.74}{0.920} & \textcolor[rgb]{0.0, 0.53, 0.74}{0.966} & \textcolor[rgb]{0.0, 0.53, 0.74}{0.948} & \textcolor[rgb]{0.0, 0.53, 0.74}{0.977} & \textcolor[rgb]{0.0, 0.53, 0.74}{0.839} & \textcolor[rgb]{0.0, 0.53, 0.74}{0.922} & \textcolor[rgb]{0.0, 0.53, 0.74}{\textbf{0.966}} & \textcolor[rgb]{0.0, 0.53, 0.74}{\textbf{0.994}}\\
          \midrule
           \multirow{4}{*}{Entity-30} & ResNet18 &0.964 &0.979 &0.570 &0.836 &0.553 &0.832 &0.542 &0.935 &0.611 &0.845 &0.849 &0.978 &0.929 &0.968 &0.952 &0.987 &\textbf{0.970} &\textbf{0.995}\\
          & ResNet50 &\textbf{0.961} &0.980 &0.878 &0.969 &0.838 &0.956 &0.914 &0.975 &0.924 &0.973 &0.835 &0.956 &0.783 &0.914 &0.937 &0.986 &0.957 &\textbf{0.996}\\
          & WRN-50-2 &0.940 &0.978 &0.897 &0.974 &0.878 &0.970 &0.826 &0.955 &0.936 &0.984 &0.927 &0.973 &0.927 &0.973 &\textbf{0.959} &0.986 &0.949 &\textbf{0.994}\\
          \cline{2-20}
          & \textcolor[rgb]{0.0, 0.53, 0.74}{Average} & \textcolor[rgb]{0.0, 0.53, 0.74}{0.955} & \textcolor[rgb]{0.0, 0.53, 0.74}{0.978} & \textcolor[rgb]{0.0, 0.53, 0.74}{0.781} & \textcolor[rgb]{0.0, 0.53, 0.74}{0.926} & \textcolor[rgb]{0.0, 0.53, 0.74}{0.756} & \textcolor[rgb]{0.0, 0.53, 0.74}{0.919} & \textcolor[rgb]{0.0, 0.53, 0.74}{0.728} & \textcolor[rgb]{0.0, 0.53, 0.74}{0.956} & \textcolor[rgb]{0.0, 0.53, 0.74}{0.823} & \textcolor[rgb]{0.0, 0.53, 0.74}{0.934} & \textcolor[rgb]{0.0, 0.53, 0.74}{0.871} & \textcolor[rgb]{0.0, 0.53, 0.74}{0.969}& \textcolor[rgb]{0.0, 0.53, 0.74}{0.880} & \textcolor[rgb]{0.0, 0.53, 0.74}{0.952}  & \textcolor[rgb]{0.0, 0.53, 0.74}{0.949} & \textcolor[rgb]{0.0, 0.53, 0.74}{0.987}  & \textcolor[rgb]{0.0, 0.53, 0.74}{\textbf{0.959}} & \textcolor[rgb]{0.0, 0.53, 0.74}{\textbf{0.995}}\\
          \midrule
           \multirow{4}{*}{Living-17}  & ResNet18 &0.876 &0.973 &0.913 &0.973 &0.898 &0.970 &0.586 &0.736 &0.940 &0.973 &0.768 &0.950 &0.900 &0.958 &0.923 &0.970  &\textbf{0.949} &\textbf{0.983}\\
          & ResNet50 &0.906 &0.956 &0.880 &0.967 &0.853 &0.961 &0.633 &0.802 &\textbf{0.938} &\textbf{0.976} &0.771 &0.926 &0.851 &0.929 &0.903 &0.924 &0.931 &0.975\\
          & WRN-50-2 &0.909 &0.957 &0.928 &0.980 &0.921 &0.977 &0.652 &0.793 &\textbf{0.966} &\textbf{0.984} &0.931 &0.967 &0.931 &0.966 &0.915 &0.970 &0.910 &0.976\\
          \cline{2-20}
          & \textcolor[rgb]{0.0, 0.53, 0.74}{Average} & \textcolor[rgb]{0.0, 0.53, 0.74}{0.933} & \textcolor[rgb]{0.0, 0.53, 0.74}{0.974} & \textcolor[rgb]{0.0, 0.53, 0.74}{0.907} & \textcolor[rgb]{0.0, 0.53, 0.74}{0.973} & \textcolor[rgb]{0.0, 0.53, 0.74}{0.814} & \textcolor[rgb]{0.0, 0.53, 0.74}{0.969} & \textcolor[rgb]{0.0, 0.53, 0.74}{0.623} & \textcolor[rgb]{0.0, 0.53, 0.74}{0.777} & \textcolor[rgb]{0.0, 0.53, 0.74}{\textbf{0.948}} & \textcolor[rgb]{0.0, 0.53, 0.74}{\textbf{0.978}} & \textcolor[rgb]{0.0, 0.53, 0.74}{0.817} & \textcolor[rgb]{0.0, 0.53, 0.74}{0.949} & \textcolor[rgb]{0.0, 0.53, 0.74}{0.894} & \textcolor[rgb]{0.0, 0.53, 0.74}{0.951} & \textcolor[rgb]{0.0, 0.53, 0.74}{0.913} & \textcolor[rgb]{0.0, 0.53, 0.74}{0.969} & \textcolor[rgb]{0.0, 0.53, 0.74}{0.930} & \textcolor[rgb]{0.0, 0.53, 0.74}{\textbf{0.978}}\\
          \midrule
           \multirow{4}{*}{Nonliving-26}  & ResNet18 &0.906 &0.955 &0.781 &0.925 &0.739 &0.909 &0.543 &0.810 &0.854 &0.939 &0.914 &0.980 &\textbf{0.958} &0.981 &0.939 &0.978 &0.953 &\textbf{0.983}\\
          & ResNet50 &0.916 &0.970 &0.832 &0.942 &0.776 &0.918 &0.638 &0.837 &0.893 &0.960 &0.848 &0.950 &0.805 &0.907 &0.873 &0.972  &\textbf{0.945} &\textbf{0.989}\\
          & WRN-50-2 &0.917 &0.977 &0.932 &0.971 &0.912 &0.959 &0.676 &0.861 &\textbf{0.945} &0.969 &0.885 &0.942 &0.893 &0.939  &0.924 &0.973 &0.937 &\textbf{0.985}\\
          \cline{2-20}
          & \textcolor[rgb]{0.0, 0.53, 0.74}{Average} & \textcolor[rgb]{0.0, 0.53, 0.74}{0.913} & \textcolor[rgb]{0.0, 0.53, 0.74}{0.967} & \textcolor[rgb]{0.0, 0.53, 0.74}{0.849} & \textcolor[rgb]{0.0, 0.53, 0.74}{0.946} & \textcolor[rgb]{0.0, 0.53, 0.74}{0.809} & \textcolor[rgb]{0.0, 0.53, 0.74}{0.929} & \textcolor[rgb]{0.0, 0.53, 0.74}{0.618} & \textcolor[rgb]{0.0, 0.53, 0.74}{0.836} & \textcolor[rgb]{0.0, 0.53, 0.74}{0.897} & \textcolor[rgb]{0.0, 0.53, 0.74}{0.956} & \textcolor[rgb]{0.0, 0.53, 0.74}{0.882} & \textcolor[rgb]{0.0, 0.53, 0.74}{0.957}& \textcolor[rgb]{0.0, 0.53, 0.74}{0.913}& \textcolor[rgb]{0.0, 0.53, 0.74}{0.974} & \textcolor[rgb]{0.0, 0.53, 0.74}{0.886} & \textcolor[rgb]{0.0, 0.53, 0.74}{0.943} & \textcolor[rgb]{0.0, 0.53, 0.74}{\textbf{0.945}} & \textcolor[rgb]{0.0, 0.53, 0.74}{\textbf{0.985}}\\
         \bottomrule
·    \end{tabular}}}
    \vspace{-20pt}
    \label{tab:main 1}
\end{table*}

% \begin{figure*}[!t]
%     \centering
%     \includegraphics[width=1.\linewidth]{icml2024/imgs/pseudo_labels.pdf}
%     \vspace{-10pt}
%     \caption{Performance comparison on different label generation strategies including our label generation strategy (i.e., \textit{Ours}), using pseudo labels only (i.e., \textit{Pseudo labels}), and using true labels only (i.e., \textit{True labels}) with ResNet18, ResNet50 and WRN-50-2.}
%     \label{fig:label_generation}
%     % \vspace{-10pt}
% \end{figure*}

% \begin{figure*}[h]
%     \centering
%     \includegraphics[width=1.\linewidth]{icml2024/imgs/loss_selection.pdf}
%     \vspace{-10pt}
%     \caption{Performance comparison on different types of losses including the cross-entropy loss (i.e., \textit{ours}), the entropy loss for samples with low confidence (i.e., \textit{Entropy}), and the cross-entropy loss with label smoothing (i.e., \textit{Smoothing}) with ResNet18, ResNet50, and WRN-50-2.}
%     \label{fig:loss_selection}
%     \vspace{-10pt}
% \end{figure*}

\paragraph{Test datasets.} In our comprehensive evaluation, we consider 11 datasets with 3 types of distribution shifts: synthetic, natural, and novel subpopulation shift. To verify the effectiveness of our method under the synthetic shift, we use CIFAR-10C, CIFAR-100C, and ImageNet-C \citep{hendrycks2019benchmarking} that span 19 types of corruption across 5 severity levels, as well as TinyImageNet-C \citep{hendrycks2019benchmarking} with 15 types of corruption and 5 severity levels. For the natural shift, we use the domains excluded from training from Office-31, Office-Home, and Camelyon17-WILDS as the OOD datasets. For the novel subpopulation shift, we consider the BREEDS benchmarks, namely, Living-17, Nonliving-26, Entity-13, and Entity-30, which are constructed from ImageNet-C.

% \textbf{Training details.} To obtain the pre-trained model, we consider 3 types of model structures, ResNet18, ResNet50 \cite{he2016deep} and WRN-50-2 \cite{zagoruyko2016wide}, with 20 epochs for CIFAR-10 \cite{krizhevsky2009learning} and 50 epochs for the other datasets. We train those models via SGD with a learning rate of $10^{-3}$, cosine learning rate decay \cite{loshchilov2016sgdr}, a momentum of 0.9 and the batch size of 128. 

\paragraph{Training details.} To show the versatility of our approach across different architectures, we perform all our experiments on ResNet18, ResNet50 \citep{he2016deep} and WRN-50-2 \citep{zagoruyko2016wide} models. We train them for 20 epochs for CIFAR-10 \citep{krizhevsky2009learning} and 50 epochs for the other datasets. In all cases, we use SGD with a learning rate of $10^{-3}$, cosine learning rate decay \citep{loshchilov2016sgdr}, a momentum of 0.9, and a batch size of 128. For all experiments, we use $p=0.3$ to compute \textsc{GdScore}.

% \textbf{Evaluation matrics.} To measure how strong the linear relationship is between the designed score and the ground-truth OOD error, we use coefficients of determination ($R^{2}$) and Spearman correlation coefficients ($\rho$) as the evaluation metrics. 
% To compare the computational efficiency with two self-training methods, we calculate the average evaluation time ($T$) consumed for every test dataset.

\paragraph{Evaluation metrics.} We measure the performance of all competing methods using the coefficients of determination ($R^{2}$) and Spearman correlation coefficients ($\rho$) calculated between the baseline scores and the true test error. To compare the computational efficiency with two self-training methods, we calculate the average evaluation time needed for every test dataset.

% \textbf{Baseline methods.} We consider 8 baselines to compare against our methods: \textit{Rotation Prediction} (Rotation) \cite{deng2021does}, \textit{Averaged Confidence} (ConfScore) \cite{hendrycks2016baseline}, \textit{Entropy} \cite{guillory2021predicting}, \textit{Agreement Score} (AgreeScore) \cite{jiang2021assessing}, \textit{Averaged Threshold Confidence} (ATC) \cite{garg2022leveraging}, \textit{AutoEval} (Fr\'{e}chet) \cite{deng2021labels}, \textit{Dispersion Score} (Dispersion) \cite{xie2023importance} and \textit{ProjNorm} \cite{yu2022predicting}, where the first seven methods are training-free and the last method belongs to the self-training approach. In general, Rotation and Agreescore address the problem by constructing an unsupervised loss. ConfScore, Entropy and ATC estimate the OOD error using the model predictions. Fr\'{e}chet and ProjNorm gauge the distribution discrepancy as the OOD error estimation from the feature-representation level and the parameter level, and Dispersion measures the separability of the test feature representation. More detials are introduced in Appendix TODO. 

\paragraph{Baselines.} We compare our method \textsc{GdScore} with $8$ baselines commonly considered in the unsupervised accuracy estimation literature: \textit{Rotation Prediction} (Rotation) \citep{deng2021does}, \textit{Averaged Confidence} (ConfScore) \citep{hendrycks2016baseline}, \textit{Entropy} \citep{guillory2021predicting}, \textit{Agreement Score} (AgreeScore) \citep{jiang2021assessing}, \textit{Averaged Threshold Confidence} (ATC) \citep{garg2022leveraging}, \textit{AutoEval} (Fr\'{e}chet) \citep{deng2021labels}, \textit{Dispersion Score} (Dispersion) \citep{xie2023importance}, and \textit{ProjNorm} \citep{yu2022predicting}. The first six methods are training-free and the last method is an instance of self-training approaches. More details about the baselines can be found in Appendix~\ref{appendix_baselines}. 
% In general, Rotation and Agreescore address the problem by constructing an unsupervised loss. ConfScore, Entropy, ATC and Nuclear Norm estimate the test accuracy using the model predictions. Fr\'{e}chet and ProjNorm gauge the distribution discrepancy as the test accuracy estimation from the feature-representation level and the parameter level, and Dispersion measures the separability of the test feature representation. 

\subsection{Main takeaways}
% \textbf{GrdNorm Score is superior to existing methods based on their estimation performance.} In Table \ref{tab:main 1}, we present the OOD error estimation performance on 11 benchmark datasets and 3 types of model structures measured by $R^2$ and $\rho$. We observe that GrdNorm Score dramatically outperforms existing methods under diverse distribution shifts. For example, our method achieves an average $R^2$ higher than 0.990 on CIFAR-100, while the average $R^2$ of the other baseline methods is all below that value. In addition, our method performs stably across different distribution shifts compared with the other existing algorithms. For example, Rotation performs well under the natural shift, but experiences a dramatic performance drop under the synthetic shift, ranking from the second best to the eighth. But our method achieves consistently high performance, ranking the best on average across the two types of distribution shifts. 
\paragraph{\textsc{GdScore} correlates with test accuracy stronger than baselines across diverse distribution shifts.} In Table \ref{tab:main 1}, we present the OOD error estimation performance on 11 benchmark datasets across 3 model architectures as measured by $R^2$ and $\rho$. We observe that \textsc{GdScore} outperforms existing methods under diverse distribution shifts. Our method achieves an average $R^2$ higher than $0.99$ on CIFAR-100, while the average $R^2$ of the other baselines is always below. In addition, our method performs stably across different distribution shifts compared with the other existing algorithms. For example, Rotation performs well under the natural shift but experiences a dramatic performance drop under the synthetic shift, ranking from the second best to the eighth. However, our method achieves consistently high performance, ranking the best on average across the three types of distribution shifts.

\begin{figure*}[!t]
    \centering
    % \hfill
    \subfigure[Label generation strategies]{
        % \centering
\includegraphics[width=0.5\textwidth, height=4cm]{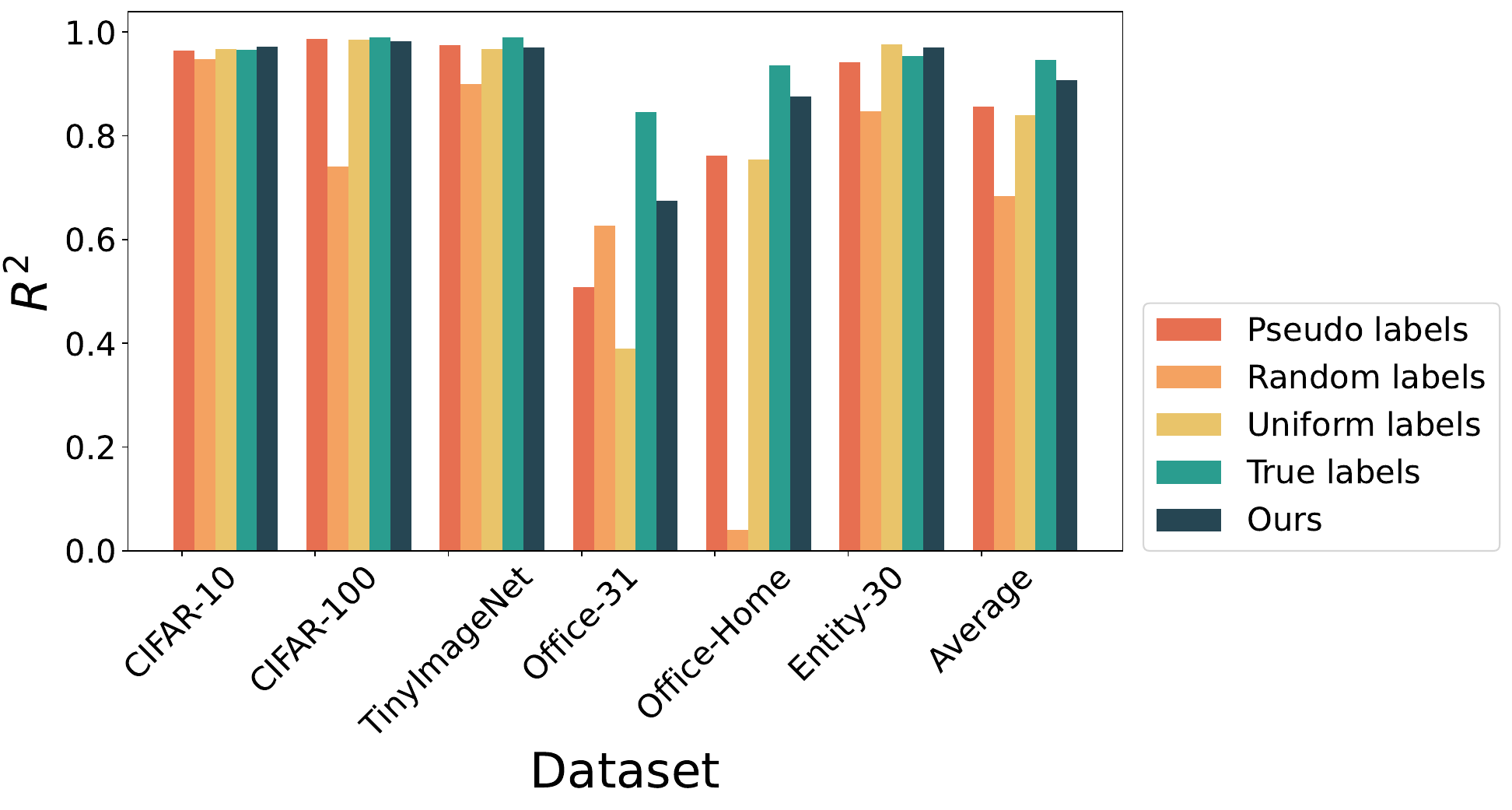}
        \label{fig:label_generation}}
    % \hfill
    \subfigure[Loss selection]{
        % \centering
\includegraphics[width=0.4\textwidth, height=4cm]{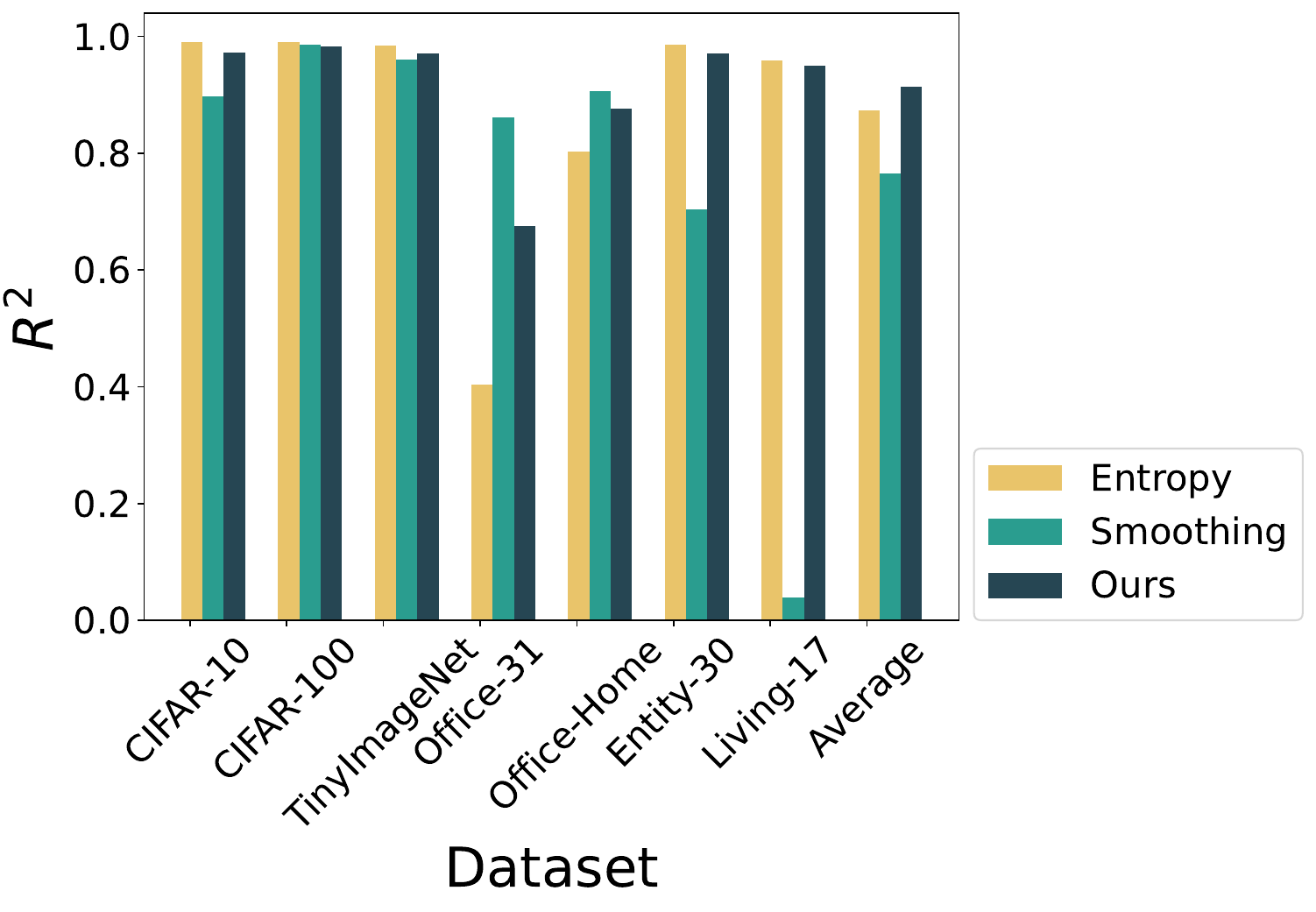}
        \label{fig:loss_selection}}
    \caption{Performance comparison ($R^2$) on $7$ datasets with ResNet18 between (a) different label generation strategies and (b) different types of losses across $3$ types of distribution shifts. Results confirm that our proposed method performs better on average across various datasets and types of shifts.}
     %\caption{Performance comparison ($R^2$) on \ref{fig:label_generation} different label generation strategies and \ref{fig:loss_selection} different types of losses across 3 types of distribution shifts with ResNet18.}
     \vspace{-15pt}
\end{figure*}

Furthermore, we provide the visualization of estimation performance in Fig. \ref{fig:scatters}, where we present the scatter plots for Dispersion Score, ProjNorm and \textsc{GdScore} on Entity-13 with ResNet18. We can see that \textsc{GdScore} and test accuracy have a strong linear relationship, while the other state-of-the-art methods struggle to have a linear correlation in cases when the test error is high. This phenomenon demonstrates the superiority of \textsc{GdScore} in unsupervised accuracy estimation. In the next paragraph, we also demonstrate the computational efficiency of our approach.

\paragraph{\textsc{GdScore} is a more efficient self-training approach.}  In the list of baselines, both our method and ProjNorm \citep{yu2022predicting} belong to self-training methods~\citep{amini2022self}. The latter, however, requires costly iterative training on the neural network during evaluation to obtain the complete set of fine-tuned parameters to calculate the distribution discrepancy in network parameters. Compared with ProjNorm, our method only trains the model for one epoch and collects the gradients of the linear classification layer to calculate the gradient norm, which is much more computationally efficient. Fig. \ref{fig:time} presents the comparison of computational efficiency between the two methods on 7 datasets with ResNet50. From this figure, we can see that our method is up to 80\% faster than ProjNorm on average. This difference is striking on the Office-31 dataset, where our method is not only two orders of magnitude faster than ProjNorm but also improves the $R^2$ score by a factor of 2.

\begin{wrapfigure}{r}{0.5\textwidth}
    \centering
    \vspace{-10pt}
    \includegraphics[width=0.75\linewidth]{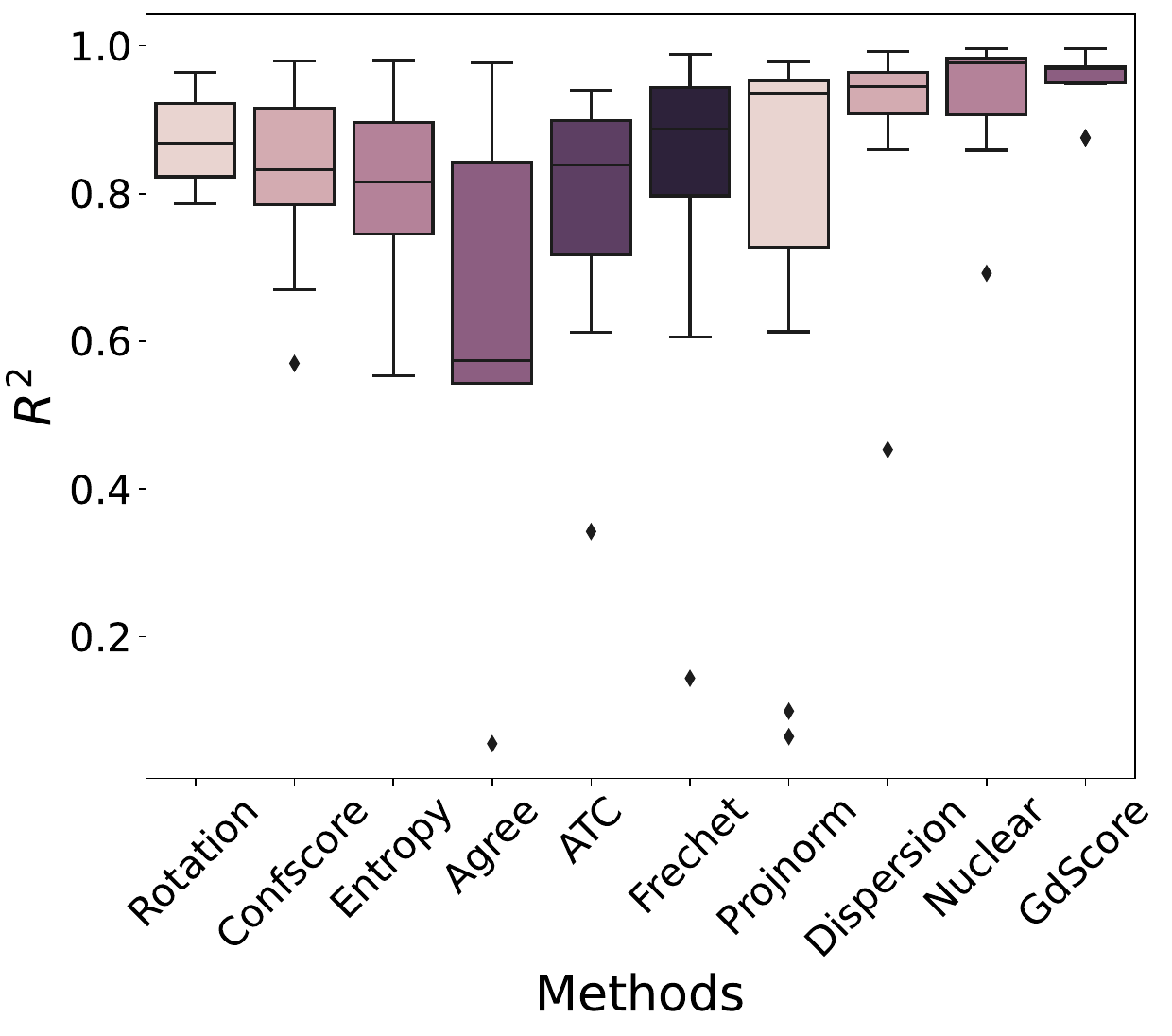}
    \caption{Robustness comparison for all estimation baselines across diverse distribution shifts with ResNet18.}
    \label{fig:boxplot}
    \vspace{-40pt}
\end{wrapfigure}
\paragraph{Robustness of our approach.} \textsc{GdScore} achieves great performance across all datasets, architectures, and types of shifts (Table~\ref{tab:main 1}). To highlight the robustness of our approach, we compare the distributions of $R^2$ in Figure~\ref{fig:boxplot}, including as additional baseline the very recent \textit{Nuclear Norm} \citep{deng2023confidence} (more results in Appendix~\ref{appendix_nuclear}). We show that \textsc{GdScore} is the best and most stable approach on $10$ datasets (except ImageNet). %in Figure~\ref{fig:boxplot}, where we find that our method is more stable than the other baselines with qualified estimation performance on 10 datasets except ImageNet. We also provide partial numerical results in Appendix~\ref{appendix_nuclear}.

% Furthermore, we also provide the visualization of estimation performance in Fig. \ref{fig:scatters}, where we present the scatter plots for Dispersion Score, ProjNorm and GrdNorm Score on Entity-13 with ResNet18. From this figure, we find that GrdNorm Score and OOD error perform a strong linear relationship, while the other state-of-the-art methods fail especially when the OOD error is high. This phenomenon clearly demonstrates the superiority of GrdNorm Score in OOD error estimation.

\section{Ablation study} \label{section_ablation}
% In this section, we motivate our implementation choices with a thorough ablation study.
%Our method involves several algorithmic choices such as the proposed pseudo-labeling strategy, the loss used for backpropagation as well as the number of gradient steps, and the $p$ values of the norm used. Below, we provide an ablation study related to all these choices. 
%\subsection{Pseudo-labeling} \label{section_label_generaration}

In this section, we will conduct comprehensive ablation studies to verify the efficiency of our proposed label-generation strategy, and choice of hyperparameters such as $\tau$ and $p$ as well as self-training settings to obtain GdScore such as epochs, blocks, and losses.

\paragraph{Random pseudo-labels boost the performance under natural shift.} In Fig. \ref{fig:label_generation}, we conduct an ablation study to verify the effectiveness of our label generation strategy by comparing it with ground-truth labels, uniform labels \citep{huang2021importance}, full random labels, and full pseudo-labeling. From the figure, we observe that while comparable with the ground truth under the synthetic drift and the subpopulation shift, the other strategies lead to a drastic drop in performance under the natural shift. This phenomenon is possibly caused by imprecise gradient calculation based on incorrect pseudo labels. Our labeling strategy performs better on average suggesting that random labels for low-confidence samples provide certain robustness of the score under natural shift.

\begin{figure*}[!t]
    \centering
    % \hfill
    \subfigure[Norm types]{
        \centering
        \includegraphics[width=0.3\textwidth]{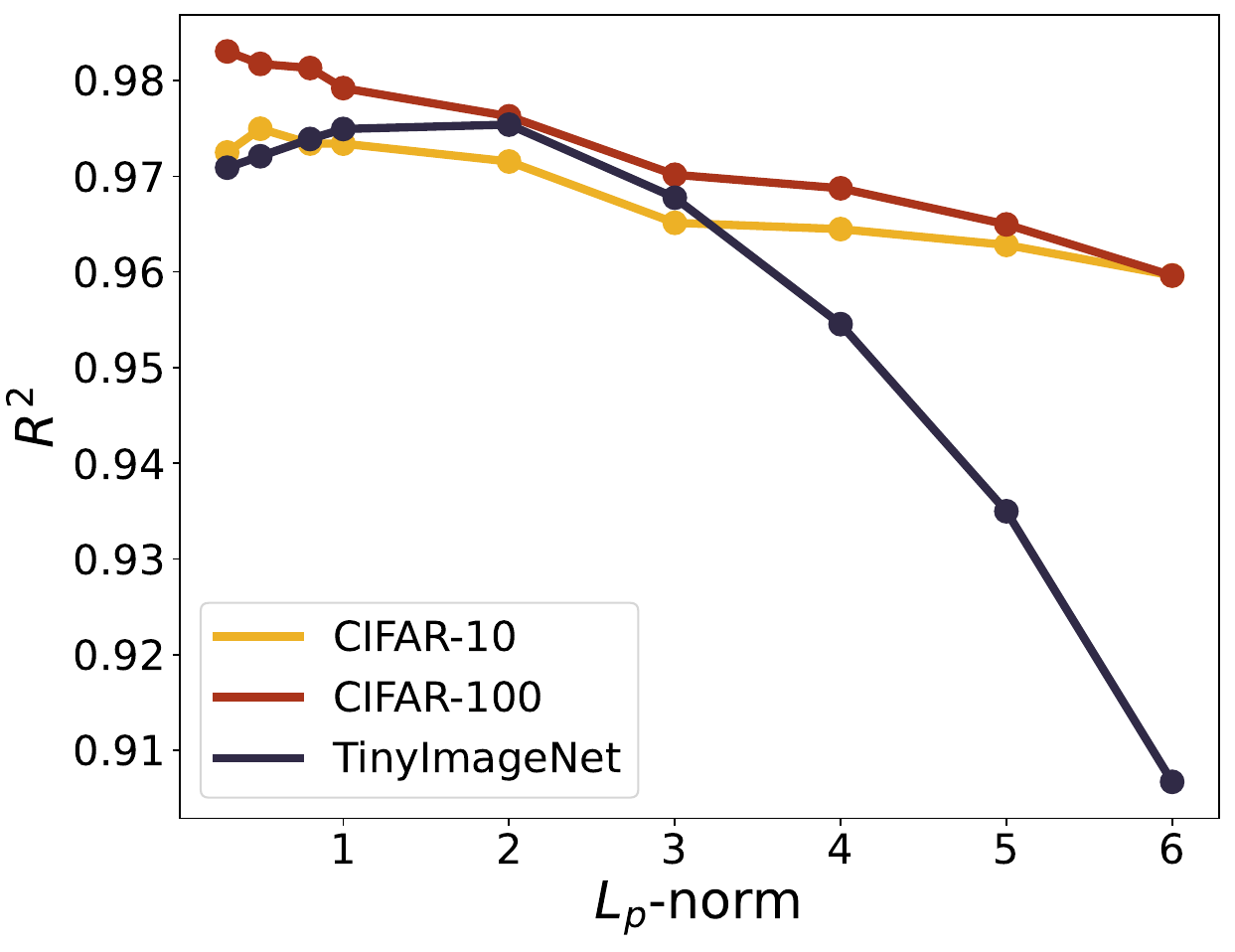}
        \label{fig:norm_types}}
    % \hfill
    \subfigure[Layer selection]{
        \centering
        \includegraphics[width=0.3\textwidth]{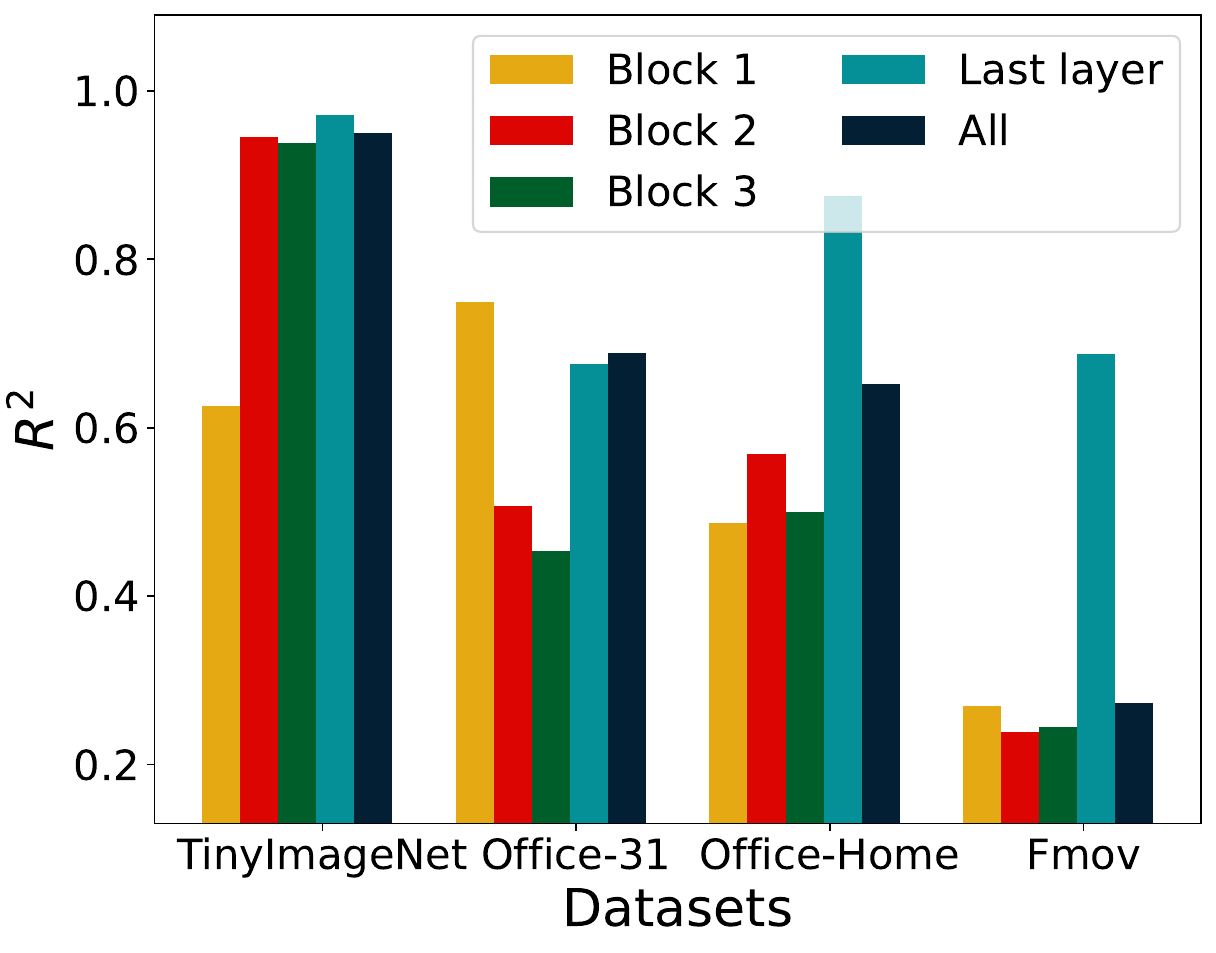}
        \label{fig:layer_selection}}
    % \hfill
     \subfigure[Epoch selection]{
        \centering
        \includegraphics[width=0.3\textwidth]{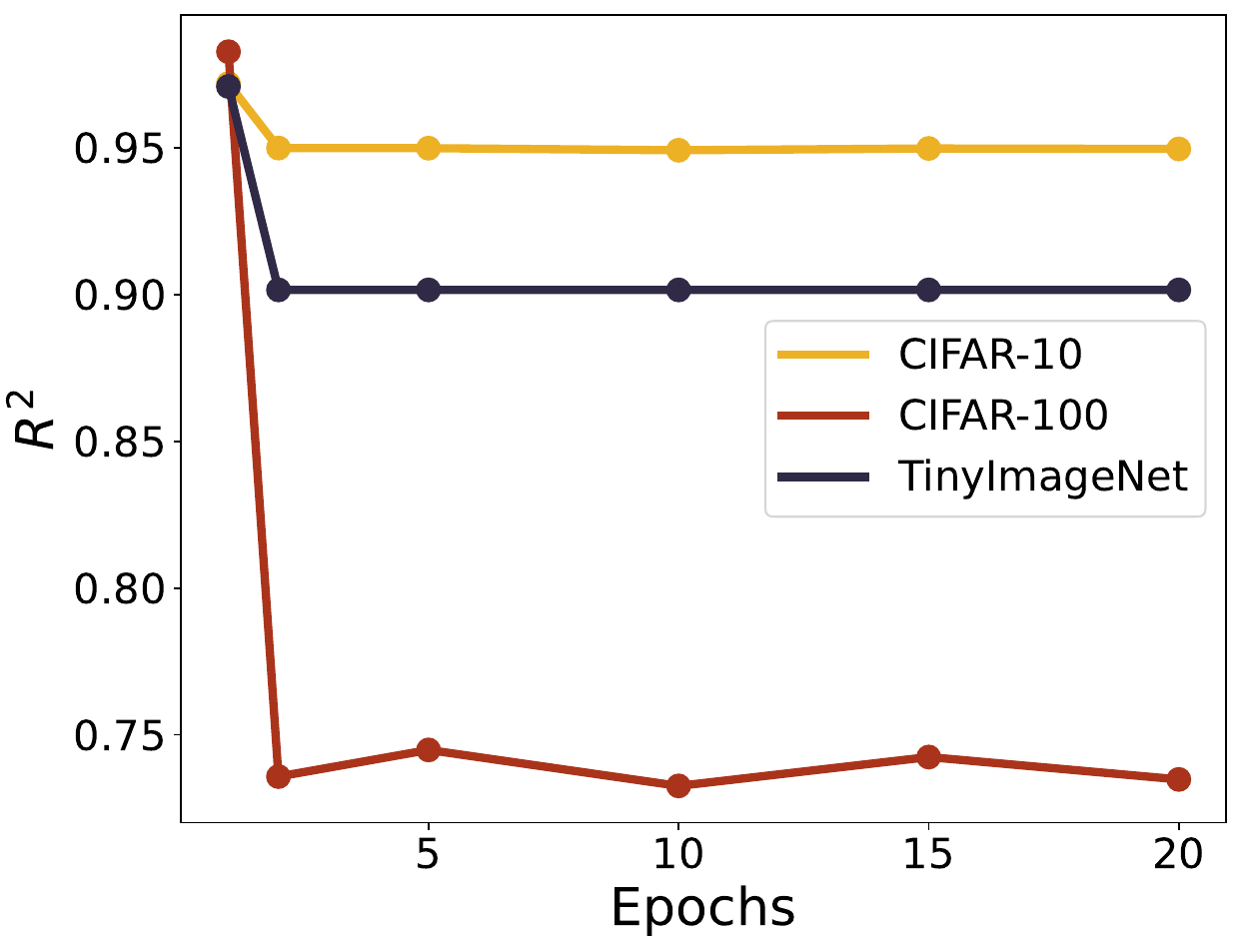}
        \label{fig:epochs}}
     \caption{Sensitivity analysis on the effect of (a) norm types, (b) layer selection for gradients, and (c) epoch selection. The first and the third experiments are conducted on CIFAR-10C, CIFAR-100C, and TinyImageNet-C, while the second experiment includes TinyImageNet-C, Office-31, Office-Home, and WILDS-FMoV \citep{koh2021wilds}. All experiments are conducted with ResNet18.}
     \vspace{-15pt}
\end{figure*}

\paragraph{Cross-entropy loss is robust to different shifts.} To demonstrate the impact of different losses on the test accuracy estimation performance, we compare the standard cross-entropy loss used in our method with the entropy loss for samples with low confidence (see definition in Appendix~\ref{appendix_loss_selection}). 
% In particular, the detail about the entropy loss for samples with low confidence is expressed as follows:
% In particular, the entropy loss can be expressed as follows:
% \begin{align*}
%     \mathcal{L}(f_{\bm{\theta}}(\Tilde{\mbf{x}})) &= -\frac{1}{m_1}\sum_{i=1}^{m_1} \sum_{k=1}^K \Tilde{y}_{i, con>\tau}^{(k)} \log \mathrm{s}_{\bm{\omega}}^{(k)}(f_{\bm{g}}(\Tilde{x}_i^{con>\tau})) \\
%     & - \frac{1}{m_2}\sum_{i=1}^{m_2}\sum_{k=1}^K \mathrm{s}_{\bm{\omega}}^{(k)}(f_{\bm{g}}(\Tilde{\mbf{x}}_i^{con \leq \tau})) \log \mathrm{s}_{\bm{\omega}}^{(k)}(f_{\bm{g}}(\Tilde{\mbf{x}}_i^{con \leq \tau})),
% \end{align*}
% where the first term denotes the cross-entropy loss calculated for samples with confidence larger than the threshold value $\tau$, the second term denotes the entropy loss for samples with lower confidence than $\tau$, and $con$ means the sample confidence, $m_1$ and $m_2$ denote the total number of samples with higher confidence and lower confidence than $\tau$, respectively. 
Moreover, we verify the effectiveness of the label smoothing, a simple yet effective tool for model calibration and performance improvement \citep{muller2019does}, by setting the smoothing rate as 0.4.
% On the other hand, since label smoothing is a simple yet efficient tool for classification performance improvement and model calibration \citep{muller2019does}, we also conduct an ablation study to verify the effectiveness of this regularization by setting the smoothing rate as 0.4.

% Its detailed introduction is in Appendix~\ref{appendix_loss_selection}.
% 

% Fig. \ref{fig:loss_selection} illustrate their performance comparison of OOD error estimation. From the figure, we find that the standard cross-entropy loss is the robustest loss across different types of distribution shifts compared to the other two losses, while the entropy loss enhances the estimation performance under the synthetic shift and the the novel subpopulation shift, but dramatically damages the model under the natural shift. On the contrary, the label smoothing regularization can increase the estimation performance under the natural shift, but decrease the performance under the synthetic shift and the the novel subpopulation shift.

% Since label smoothing is a simple yet efficient tool for classification performance improvement and model calibration \cite{muller2019does}, we also conduct an ablation study to verify the effectiveness of this regularization by setting the smoothing rate as 0.4. 
Fig. \ref{fig:loss_selection} illustrates this comparison revealing that standard cross-entropy loss is the most robust choice across different types of distribution shifts. We also note that the entropy loss enhances the estimation performance under the synthetic and the novel subpopulation shifts, but struggles under the natural shift. On the contrary, the label smoothing regularization can increase the performance under the natural shift but decrease it under the synthetic and the novel subpopulation shifts.

\paragraph{Choosing smaller $p$ for better estimation performance.} To illustrate the effect of the choice of $L_p$-norm on the performance, we conduct a sensitivity analysis on three datasets with ResNet18 summarized in Figure \ref{fig:norm_types}. We can see that there is an obvious decreasing trend as $p$ becomes larger. Especially, when $p$ is smaller than 1, the estimation performance can fluctuate within a satisfying range. This is probably because a smaller $p$ (i.e., $0<p<1$) makes the $L_p$-norm more suitable for the high-dimensional space, while the $L_p$-norm with $p \geq 1$ is likely to ignore gradients that are close to 0 \citep{wang2016diversity, huang2021importance}, so for all experiments, we used $p=0.3$. %So in this paper, we set $p$ as 0.3 for all datasets and model structures.
% This is possibly because a smaller $p$ emphasizes the small elements in the gradient vector rather than the larger elements \citep{huang2021importance}, so that 
The remark below, whose proof we defer to Appendix~\ref{app:p_lower_than_one}, provides some theoretical insights into the $L_p$-norm of the gradient for $0<p<1$.
\begin{rmk}[Case $0<p<1$]
\label{rmk:p_lower_than_one}
    Let $\mbf{c}$ be the classifier obtained from $\bm{\omega}_s$ after one gradient descent step, i.e., $\mbf{c}=\bm{\omega}_s-\eta\cdot \nabla \mathcal{L}_T(\bm{\omega}_s)$ with $\eta \geq 0$. For any $p \in (0,1)$, we have
    \begin{equation*}
        \eta \lVert \nabla \mathcal{L}_T(\bm{\omega}_s) \rVert_p \leq \lvert \lVert \mbf{c} \rVert_p - \lVert \bm{\omega}_s \rVert _p \rvert.
    \end{equation*}
\end{rmk}
% \subsection{Layer selection} 
\paragraph{Gradients from the last layer can provide sufficient information.} In this part, we aim to understand whether backpropagation through other layers in the neural network can provide a better test accuracy estimate. For this, we separate the feature extractor $f_{\mbf{g}}$ into 3 blocks of layers with roughly equal size and calculate the \textsc{GrdNorm} scores for each of them. Additionally, we also try to gather the gradients over the whole network. Fig. \ref{fig:layer_selection} plots the obtained results on 4 datasets, suggesting that the last layer provides sufficient information to predict the true test accuracy in an unsupervised way.

% \subsection{Epoch selection}
\paragraph{No gains after 1 epoch of backpropagation.} Here, we train the neural network for $r$ epochs, where $r \in \{1, 2, 5, 10, 15, 20\}$, and store the gradient vectors of the classification layer for each value of $r$. Fig. \ref{fig:epochs} suggests that the gradient norms after 1 step are sufficient to predict the model performance under distribution shifts. Further training gradually degrades the performance with the increasing $r$. The reason behind the phenomenon is that the gradients in the first epoch contain the most abundant information about the training dataset, while the neural network fine-tuned on the test dataset for several epochs is going to forget previous training categories  \citep{kemker2018measuring}.

\paragraph{Choice of proper threshold $\tau$}
\label{app:choice_threshold}
In our experiments (see Section~\ref{section:experiments}), we set the value of $\tau$ as 0.5 across all datasets and network architectures. This choice of $\tau$ is due to the intuition that if a label contains a softmax probability below 0.5, it means that this predicted label has over $50\%$ chances of being wrong. It means that this label has a higher probability of being incorrect than to be correct. Thus, we tend to regard it as an incorrect prediction. To demonstrate the impact of threshold $\tau$ on the final performance, we conduct an ablation study on CIFAR-10C and Office-31 with ResNet18 using varying values of $\tau$. We display in Table~\ref{tab:threshold} the corresponding values of $R^2$. We can observe that the final performance improves and achieves its best value for $\tau$ is $0.5$, before decreasing slightly.

\begin{table}[!ht]
    \centering
    \caption{Performance on CIFAR-10 and Office-31 with ResNet18 for varying value of $\tau$. The metric used in this table is the coefficient of determination $R^2$. The best result is highlighted in \textbf{bold}.}
    %\resizebox{\textwidth}{!}{
    \begin{tabular}{c|cccccccccc}
    \toprule
        Threshold &0.0&0.1&0.2&0.3&0.4&0.5&0.6&0.7&0.8&0.9\\
        \midrule
         CIFAR-10C & 0.963 &0.963 &0.964 &0.965 &0.971 &\textbf{0.972} &0.967& 0.962&0.963&0.959\\
        \midrule
         Office-31 &0.495 &0.498	&0.532	&0.674	&\textbf{0.685}	&0.667	&0.545	&0.451	&0.114	&0.131\\
    \bottomrule
    \end{tabular}
    %}
    \label{tab:threshold}
\end{table}

\section{Connection to GradNorm}
\label{app:comparison_grad_norm}
A current work, GradNorm \citep{huang2021importance}, employs gradients to detect OOD samples whose labels belong to a different label space from the training data. It gauges the magnitude of gradients in the classification layer, backpropagated from a KL divergence between the softmax probability and uniform distribution. Compared with GradNorm, our method bears three critical differences, in terms of the problem setting, methodology, and theoretical insights.
% We also empirically demonstrate the superiority of our method in Table~\ref{tab:gradnorm}. 
% As an OOD detection method, GradNorm \citep{huang2021importance} provides sufficient insights for classification with rejection from the view of gradients. It gauges the magnitude of gradients in the classification layer, backpropagated from a KL-divergence 

1) \textit{Problem setting}: GradNorm focuses on OOD detection, which aims to determine whether a given sample is in-distribution (ID) or out-of-distribution \citep{hendrycks2016baseline, hendrycks2018deep, liu2020energy, yang2021generalized, liang2017enhancing}, while our method aims to estimate the test accuracy without ground-truth test labels. It requires GradNorm should be an instance-level score for classification, but our method is a dataset-level score for linear regression. Furthermore, in OOD detection, the label spaces of OOD data and training data are disjoint, while in unsupervised test accuracy, the training and the OOD label spaces are shared. The two are also evaluated differently: AUROC score for OOD detection and correlation coefficients for error estimation. Those differences are summarized in Table~\ref{tab:ood_detection_vs_ood_error_estimation}.
\begin{table}[!h]
    \centering    
    \caption{Main differences between OOD detection and OOD error estimation.}
    \begin{tabular}{c|ccc}
    \toprule
         Learning Problem &Goal&Scope&Metric\\
         \midrule
         OOD detection &Predict ID/OOD &$\Tilde{\mbf{x}_i}$ &AUROC \\ 
         OOD error estimation & Proxy to test error& $\mathcal{D}_{\text{test}}$& $R^2$ and $\rho$ \\
         \bottomrule
    \end{tabular}    
    \label{tab:ood_detection_vs_ood_error_estimation}
    \vspace{-10pt}
\end{table}

2) \textit{Methodology}: GradNorm obtains the magnitude of gradients via a KL-divergence loss measuring the distribution distance from the training distribution to a uniform distribution, while our method obtains them using a standard cross entropy loss with the specifically designed label generation strategy. GradNorm assumes that OOD data should have a lower magnitude of gradients, but in our cases, test data are certified to have a higher magnitude of gradients. Table~\ref{tab:gradnorm} presents the performance comparison of the two methods in unsupervised accuracy estimation on 7 datasets across 3 types of distribution shifts with ResNet18. It illustrates that GradNorm is inferior to our approach for unsupervised accuracy estimation suggesting that the two problems cannot be tackled with the same tools.
% Essentially, GradNorm measures the magnitude of gradients from the prediction probability to the uniform distribution, while our method measures it from the source distribution to the target distribution. This basic difference is due to different aims, and is specifically reflected in the design of losses, label generalization strategies and evaluation approaches.
% GradNorm formulates a KL-divergence as the loss for backpropagation, while our method directly utilizes the cross entropy loss. In addition, GradNorm simply assigns uniform labels 

3) \textit{Theoretical insights}: GradNorm is certified to capture the joint information between features and outputs to detect OOD data from the oncoming dataset. However, the reason why OOD data have a lower magnitude of gradients is still unclear. Our method provides clearer theoretical insights to demonstrate the relationship between gradients and test accuracy even under distribution shift, which further inspires future work to address generalization issues from the view of gradients

% two types of parameter discrepancy information that are beneficial to predicting OOD performance. Formally, we also demonstrate that our score formulates the upper bound of the true OOD error, which further explains the effectiveness of our method.
 
% In Table~\ref{tab:gradnorm}, we present the performance comparison of the two methods in OOD error estimation on 7 datasets across 3 types of distribution shifts with ResNet18. This table illustrates that our method is superior to \citet{huang2021importance}: for example, our method outperforms \citet{huang2021importance} on TinyImageNet with a large margin from 0.894 to 0.971. This shows that comparing the softmax outputs to uniform distribution as done by \citet{huang2021importance} is relevant for detecting test samples from a different label space only. However, for OOD error estimation with overlapping labels, estimating the target distribution through pseudo-labeling -- rather than assuming it to be uniform -- is more informative and achieves much better results.
\begin{table}[!ht]
    \centering
    \caption{Performance comparison between \citet{huang2021importance} and our methods on 7 datasets with ResNet18. The metric used in this table is the coefficient of determination $R^2$. The best results are highlighted in \textbf{bold}.}
    \resizebox{\textwidth}{!}{
    \begin{tabular}{c|ccccccc}
    \toprule
        Method &CIFAR 10 &CIFAR 100 &TinyImageNet &Office-31 &Office-Home &Entity-30 &Living-17\\
        \midrule
         \citep{huang2021importance} & 0.951 &0.978 &0.894 &0.596 &0.848 &0.964 &0.942 \\
         Ours & \textbf{0.972} &\textbf{0.983} &\textbf{0.971} &\textbf{0.675} &\textbf{0.876} &\textbf{0.970} &\textbf{0.949}\\
    \bottomrule
    \end{tabular}}
    \label{tab:gradnorm}
    \vspace{-10pt}
\end{table}

\section{Conclusion}
\label{conclusion}
In this paper, we showcased the strong linear relationship between the magnitude of gradients and the model performance under distribution shifts. We proposed \textsc{GdScore}, a simple yet efficient method to estimate the ground-truth test accuracy error by measuring the gradient magnitude of the last classification layer. Our method consistently achieves superior performance across various distribution shifts than previous works. Furthermore, it does not require the detailed architecture of the feature extractors and training datasets. Those properties guarantee that our method can be easily deployed in the real world, and meet practical demands, such as large models and confidential information. We hope that our research sheds new light on the usefulness of gradient norms for unsupervised accuracy estimation.

\subsubsection*{Acknowledgments}
This research is supported by the National Research Foundation Singapore and DSO National Laboratories under the AI Singapore Programme (AISGAward No: AISG2-GC-2023-009).
% Only add this information once your submission is accepted and deanonymized. 

\bibliography{main}
\bibliographystyle{tmlr}

\clearpage
\appendix
\section{Appendix}
\section{Pseudo-code of \textsc{GdScore}}
\label{app:algorithm}

Our proposed \textsc{GdScore} for unsupervised accuracy estimation can be calculated as shown in Algorithm \ref{alg:GrdNorm}.

\begin{algorithm}[h]
   \caption{Unsupervised Accuracy Estimation via \textsc{GdScore}}
   \label{alg:GrdNorm}
\begin{algorithmic}
   \STATE {\bfseries Input:} Test dataset from unseen domains $\Tilde{\mathcal{D}}=\{\Tilde{\mbf{x}}_i\}_{i=1}^{m}$, a pre-trained model $f_{\bm{\theta}}=f_{\mbf{g}}\circ f_{\bm{\omega}}$ (feature extractor $f_{\mbf{g}}$ and classifier $f_{\bm{\omega}}$), a threshold value $\tau$.
   \STATE {\bfseries Output:} The \textsc{GdScore}.
   \FOR{each test instance $\Tilde{\mbf{x}}_i$}
   \STATE Obtain the maximum softmax probability via $\Tilde{r}_i=\max_k 
\mathrm{s}_{\bm{\omega}}^{(k)}(f_{\mbf{g}}(\Tilde{\mbf{x}}_i))$ .
   \IF{$\Tilde{r}_i>\tau$}
   \STATE Obtain pseudo labels via $\Tilde{y}^{\prime}_i = \argmax_k f_{\bm{\theta}}(\Tilde{\mbf{x}}_i)$,
   \ELSE
   \STATE Obtain random labels via $\Tilde{y}^{\prime}_i \sim U[1, K]$.
   \ENDIF
   \ENDFOR
   \STATE Calculate the cross-entropy loss using assigned labels $\Tilde{y}_i$ via Eq.~\ref{eq:cross_entropy}.
   \STATE Calculate gradients of the weights in the classification layer via Eq.~\ref{eq:gradients}.
   \STATE Calculate \textsc{GdScore} $S(\Tilde{\mathcal{D}})$ via Eq.~\ref{eq:score}.
\end{algorithmic}
\end{algorithm}
\vspace{-10pt}

\section{Baselines} 
\label{appendix_baselines}
% To help reproduce the results shown in our paper, we introduce the baseline methods in detail.
\paragraph{Rotation.} \citep{deng2021does} By rotating images from both the training and the test sets with different angles, we can obtain new inputs and their corresponding labels $y^r_i$ which indicate by how many degrees they rotate. During pre-training, an additional classifier about rotation degrees should be learned. Then the \textit{Rotation Prediction} (Rotation) metric can be calculated as:
\begin{equation*}
    S_r(\mathcal{D}_{\text{test}}) = \frac{1}{m}\sum_{i=1}^m (\frac{1}{4}\sum_{r \in \{0^{\circ}, 90^{\circ}, 180^{\circ}, 270^{\circ}\}}(\mathds{1}(\hat{y}^r_i \neq y^r_i))),
\end{equation*}
where $\hat{y}^r_i$ denotes the predicted labels about rotation degrees.

\paragraph{ConfScore.} \citep{hendrycks2016baseline} This metric directly leverages the average maximum softmax probability as the estimation of the test error, which is expressed as:
\begin{equation*}
    S_{cf}(\mathcal{D}_{\text{test}}) = \frac{1}{m} \sum_{i=1}^m \max(\mathrm{s}_{\bm{\omega}}(f_{\bm{g}}(\Tilde{\mbf{x}}_i))).
\end{equation*}

\paragraph{Entropy.} \citep{guillory2021predicting} This metric estimates the test error via the average entropy loss:
\begin{equation*}
    S_{e}(\mathcal{D}_{\text{test}}) =  \frac{1}{m} \sum_{i=1}^m \sum_{k=1}^K \mathrm{s}_{\bm{\omega}}^{(k)}(f_{\bm{g}}(\Tilde{\mbf{x}}_i))\log \mathrm{s}_{\bm{\omega}}^{(k)}(f_{\bm{g}}(\Tilde{\mbf{x}}_i)).
\end{equation*}

\paragraph{AgreeScore.} \citep{jiang2021assessing} This method trains two independent neural networks simultaneously during pre-training, and estimates the test error via the rate of disagreement across the two models:

\begin{equation*}
    S_{ag}(\mathcal{D}_{\text{test}}) = \frac{1}{m}\sum_{i=1}^m \mathds{1}(\Tilde{y}'_{1,i} \neq \Tilde{y}'_{2,i}),
\end{equation*}
where $\Tilde{y}'_{1,i}$ and $\Tilde{y}'_{2,i}$ denote the predicted labels by the two models respectively. 

\paragraph{ATC.} \citep{garg2022leveraging} It measures how many test samples have a confidence larger than a threshold that is learned from the source distribution. It can be expressed as:
\begin{equation*}
    S_{atc}(\mathcal{D}_{\text{test}}) = \frac{1}{m}\sum_{i=1}^m \mathds{1}(\sum_{k=1}^K \mathrm{s}_{\bm{\omega}}^{(k)}(f_{\bm{g}}(\Tilde{\mbf{x}}_i))\log \mathrm{s}_{\bm{\omega}}^{(k)}(f_{\bm{g}}(\Tilde{\mbf{x}}_i)) < t),
\end{equation*}
where $t$ is the threshold value learned from the validation set of the training dataset.

\paragraph{Fr\'{e}chet.} \citep{deng2021labels} This method utilizes Fr\'{e}chet distance to measure the distribution gap between the training and the test datasets, which serves the test error estimation:
\begin{equation*}
    S_{fr}(\mathcal{D}_{\text{test}}) = ||\mu_{train} -\mu_{test}||+Tr(\Sigma_{train}+\Sigma_{test}-2(\Sigma_{train}\Sigma_{test})^{\frac{1}{2}}),
\end{equation*}
where $\mu_{train}$ and $\mu_{test}$ denote the mean feature vector of $\mathcal{D}$ and $\mathcal{D}_{test}$, respectively. $\Sigma_{train}$ and $\Sigma_{test}$ refer to the covariance matrices of corresponding datasets.

\paragraph{Dispersion.} \citep{xie2023importance} This paper estimates the test error by gauging the feature separability of the test dataset in the feature space:
\begin{equation*}
    S_{dis}(\mathcal{D}_{\text{test}}) = \log \frac{\sum_{k=1}^{K} m_k \cdot \|\bar{\boldsymbol{\mu}} - \Tilde{\boldsymbol{\mu}}_k\|_2^2}{K-1},
\end{equation*}
where $\boldsymbol{\mu}$ denotes the center of the whole features, and $\boldsymbol{\mu}_k$ denotes the mean of $k^{th}$-class features.

\paragraph{Nuclear Norm.} \citep{deng2023confidence} This paper estimates the test error by computing the nuclear norm of final softmax probabilities, which can be expressed as:
\begin{equation*}
    S_{nu}(\mathcal{D}_{\text{test}}) = ||P||_*,
\end{equation*}
where Nuclear norm $||P||_*$ is defined as the sum of singular values of $P$.

\paragraph{ProjNorm.} \citep{yu2022predicting} This method fine-tunes the pre-trained model on the test dataset with pseudo-labels, and measures the distribution discrepancy between the training and the test datasets in the parameter level:
\begin{equation*}
    S_{pro}(\mathcal{D}_{\text{test}}) = ||\bm{\Tilde{\theta}}_{ref}-\bm{\Tilde{\theta}}||_2,
\end{equation*}
where $\bm{\theta}_{ref}$ denotes the parameters of the pre-trained model, while $\bm{\theta}$ denotes the parameters of the fine-tuned model.

Those algorithms mentioned in this paper can be summarized as Table \ref{tab:method_summary}.

\begin{table}[!t]
    \centering    
    \caption{Method property summary including whether this method belongs to self-training or training-free approaches, and if this method requires training data or specific model architectures.}
    \begin{tabular}{c|cccc}
    \toprule
         Method &Self-training &Training-free &Training-data-free &Architecture-requirement-free \\
         \midrule
         Rotation &\XSolidBrush &\Checkmark &\Checkmark &\XSolidBrush\\
         ConfScore &\XSolidBrush &\Checkmark &\Checkmark &\Checkmark\\
         Entropy &\XSolidBrush &\Checkmark &\Checkmark &\Checkmark \\
         Agreement &\XSolidBrush &\Checkmark &\Checkmark &\XSolidBrush \\
         ATC & \XSolidBrush &\Checkmark &\XSolidBrush &\Checkmark\\
         Fr\'{e}chet &\XSolidBrush &\Checkmark &\XSolidBrush &\Checkmark\\
         Dispersion &\XSolidBrush &\Checkmark &\Checkmark &\Checkmark\\
         Nuclear &\XSolidBrush &\Checkmark &\Checkmark &\Checkmark\\
         ProjNorm & \Checkmark &\XSolidBrush &\Checkmark &\Checkmark\\
         Ours & \Checkmark &\XSolidBrush &\Checkmark &\Checkmark\\
         \bottomrule
    \end{tabular}    
    \label{tab:method_summary}
\end{table}

\section{Formulation of the entropy loss for low-confidence samples}
\label{appendix_loss_selection}
In the ablation study, we explore the impact of loss selection on the performance of unsupervised accuracy estimation. In particular, the detail about the entropy loss for samples with low confidence is expressed as follows:
In particular, the entropy loss can be expressed as follows:
\begin{align*}
    \mathcal{L}(f_{\bm{\theta}}(\Tilde{\mbf{x}})) &= -\frac{1}{m_1}\sum_{i=1}^{m_1} \sum_{k=1}^K \Tilde{y}_{i, con>\tau}^{(k)} \log \mathrm{s}_{\bm{\omega}}^{(k)}(f_{\bm{g}}(\Tilde{x}_i^{con>\tau})) \\
    & - \frac{1}{m_2}\sum_{i=1}^{m_2}\sum_{k=1}^K \mathrm{s}_{\bm{\omega}}^{(k)}(f_{\bm{g}}(\Tilde{\mbf{x}}_i^{con \leq \tau})) \log \mathrm{s}_{\bm{\omega}}^{(k)}(f_{\bm{g}}(\Tilde{\mbf{x}}_i^{con \leq \tau})),
\end{align*}
where the first term denotes the cross-entropy loss calculated for samples with confidence larger than the threshold value $\tau$, the second term denotes the entropy loss for samples with lower confidence than $\tau$, and $con$ means the sample confidence, $m_1$ and $m_2$ denote the total number of samples with higher confidence and lower confidence than $\tau$, respectively.

\section{Influence of the Calibration Error}
\label{app:calibration}
% Theoretically, the proposed pseudo-labeling strategy depends on how well the prediction probabilities are calibrated. In degraded cases, this can have a negative impact on our approach, e.g., one can imagine a model that outputs only one-hot probabilities with not a high accuracy. However, this it is generally not the case. Indeed, in practice, we do not need to have a perfectly calibrated model as we employ a mixed strategy that assigns pseudo-labels to high-confidence examples and random labels to low-confidence ones. The recent success of applying self-training models to different problems \citep{fixmatch_neurips2020, confident_anchor_neurips2021, proj_norm_icml2022} provides evidence of the suitability of the label generation strategy we adopted.

Theoretically, the proposed pseudo-labeling strategy depends on how well the prediction probabilities are calibrated. However, in practice, we do not need to have a perfectly calibrated model as we employ a mixed strategy that assigns pseudo-labels to high-confidence examples and random labels to low-confidence ones \citep{fixmatch_neurips2020, confident_anchor_neurips2021, proj_norm_icml2022}.

% When we speak of deep neural networks, which are widely accepted to be poorly calibrated, \citet{revisiting_neurips2021} showed that modern SOTA image models tend to be well-calibrated across distribution shifts. 
% The versatility of our approach makes it straightforward to combine our method with such architectures in future works. 
To demonstrate it empirically, in Table~\ref{tab:calibration-cifar}, we provide the expected calibration error (ECE, \cite{caliration_icml2017}) of ResNet18 depending on the difficulty of test data.
For this, we test first on CIFAR-10 (ID), and then on CIFAR-10C corrupted by brightness across diverse severity from 1 to 5. We can see that ECE is very low for ID data and remains relatively low across all levels of corruption severity, which shows that ResNet is quite well-calibrated on CIFAR-10.
% confirming that our label generation strategy is appropriate in our setting.

\begin{table}[!ht]
    \centering
    \caption{Expected Error Calibration (ECE) of ResNet18 on CIFAR-10 (ID) and CIFAR-10C corrupted by brightness across diverse severity from 1 to 5.}
    \begin{tabular}{c|ccccccc}
    \toprule
        Corruption Severity &ID&1&2&3&4&5&\\
        \midrule
         ECE & 0.0067 &0.0223 &0.0230 &0.0243 &0.0255 &0.0339\\
    \bottomrule
    \end{tabular}
    \label{tab:calibration-cifar}
\end{table}

On the other hand, in the case of a more complex distribution shift like Office-31, we can see that the calibration error has increased noticeably (Table \ref{tab:calibration-office-31}). It is interesting to analyze this result together with Figure \ref{fig:label_generation} of the main paper, where we compared the results between the usual pseudo-labeling strategy and the proposed one. Although our method has room for improvement compared to the oracle method, it is also significantly better than "pseudo-labels", indicating that the proposed label generation strategy is less sensitive to the calibration error.

\begin{table}[!ht]
    \centering
    \caption{Expected Error Calibration (ECE) of ResNet18 on Office-31 data set.}
    %\resizebox{\textwidth}{!}{
    \begin{tabular}{c|ccccccc}
    \toprule
        Domain & DSLR (ID) & Amazon & Webcam \\
        \midrule
         ECE & 0.2183 & 0.2167 & 0.4408 \\
    \bottomrule
    \end{tabular}
    %}
    \vspace{-10pt}
    \label{tab:calibration-office-31}
\end{table}

\section{Comparison to Nuclear Norm}
\label{appendix_nuclear}
Here, we compare our method to Nuclear Norm \citep{deng2023confidence} across 3 types of distribution shifts with ResNet18, ResNet50, and WRN-50-. Results are shown in Table~\ref{tab:nuclear}. From this table, we observe our method outperforms Nuclear Norm under synthetic shift and natural shift.

\begin{table*}[!t]
    \centering
    \caption{Performance comparison on 11 benchmark datasets with ResNet18, ResNet50 and WRN-50-2, where $R^2$ refers to coefficients of determination, and $\rho$ refers to the absolute value of Spearman correlation coefficients (higher is better). The best results are highlighted in \textbf{bold}.}
    \renewcommand\arraystretch{1.1}
    \resizebox{0.95\textwidth}{!}{
    \setlength{\tabcolsep}{1mm}{
    \begin{tabular}{cccccccccccccc}
        \toprule
        \multirow{2}{*}{Method} &\multirow{2}{*}{Network} &\multicolumn{2}{c}{CIFAR 100} &\multicolumn{2}{c}{TinyImageNet}  &\multicolumn{2}{c}{Office-Home}  &\multicolumn{2}{c}{Camelyon17} &\multicolumn{2}{c}{Entity-13} &\multicolumn{2}{c}{Entity-30}\\
        \cline{3-14}
        & &$R^2$ &$\rho$ &$R^2$ &$\rho$ &$R^2$ &$\rho$ &$R^2$ &$\rho$ &$R^2$ &$\rho$ &$R^2$ &$\rho$ \\
        \midrule
         \multirow{4}{*}{Nuclear} & ResNet18 &\textbf{0.989} &0.995 &\textbf{0.983} &\textbf{0.994}  &0.692 &0.783 &0.858 &\textbf{1.000} &\textbf{0.978} &\textbf{0.991} &\textbf{0.980} &0.993\\
          & ResNet50 &0.979 &\textbf{0.994} &0.965 &0.994 &0.731 &0.895 &0.849 &\textbf{1.000} &\textbf{0.989} &\textbf{0.996} &\textbf{0.978} &0.994\\
          & WRN-50-2 &0.962 &0.988 &0.956 &0.992 &0.766 &0.874 &0.983 &\textbf{1.000} &\textbf{0.989} &\textbf{0.995} &\textbf{0.985} &\textbf{0.996}\\
          \cline{2-14}
          & \textcolor[rgb]{0.0, 0.53, 0.74}{Average} &\textcolor[rgb]{0.0, 0.53, 0.74}{0.977} &\textcolor[rgb]{0.0, 0.53, 0.74}{0.993} &\textcolor[rgb]{0.0, 0.53, 0.74}{0.968} &\textcolor[rgb]{0.0, 0.53, 0.74}{0.993} &\textcolor[rgb]{0.0, 0.53, 0.74}{0.730} &\textcolor[rgb]{0.0, 0.53, 0.74}{0.850} &\textcolor[rgb]{0.0, 0.53, 0.74}{0.916} &\textcolor[rgb]{0.0, 0.53, 0.74}{\textbf{1.000}} &\textcolor[rgb]{0.0, 0.53, 0.74}{\textbf{0.989}} &\textcolor[rgb]{0.0, 0.53, 0.74}{\textbf{0.995}} &\textcolor[rgb]{0.0, 0.53, 0.74}{\textbf{0.989}} &\textcolor[rgb]{0.0, 0.53, 0.74}{\textbf{0.995}}\\
          \midrule
   \multirow{4}{*}{Ours} & ResNet18 &0.987 &\textbf{0.996} &0.971 &\textbf{0.994}  &\textbf{0.876} &\textbf{0.909} &\textbf{0.996} &\textbf{1.000} &0.969 &\textbf{0.991} &0.970 &\textbf{0.995}\\
          & ResNet50 &\textbf{0.991} &\textbf{0.994} &\textbf{0.980} &\textbf{0.995}  &\textbf{0.829} &\textbf{0.944} &\textbf{0.999} &\textbf{\textbf{1.000}} &0.960  &0.995 &0.957 &\textbf{0.996}\\
          & WRNt-50-2 &\textbf{0.995} &\textbf{0.998} &\textbf{0.975} &\textbf{0.996}  &\textbf{0.809} &\textbf{0.916} &\textbf{0.997} &\textbf{1.000} &0.968 &\textbf{0.995} &0.949 &0.994\\
          \cline{2-14}
          & \textcolor[rgb]{0.0, 0.53, 0.74}{Average} &\textcolor[rgb]{0.0, 0.53, 0.74}{\textbf{0.991}} &\textcolor[rgb]{0.0, 0.53, 0.74}{\textbf{0.997}} &\textcolor[rgb]{0.0, 0.53, 0.74}{\textbf{0.976}} &\textcolor[rgb]{0.0, 0.53, 0.74}{\textbf{0.995}}  &\textcolor[rgb]{0.0, 0.53, 0.74}{\textbf{0.837}} &\textcolor[rgb]{0.0, 0.53, 0.74}{\textbf{0.923}} &\textcolor[rgb]{0.0, 0.53, 0.74}{\textbf{0.998}} &\textcolor[rgb]{0.0, 0.53, 0.74}{\textbf{1.000}} &\textcolor[rgb]{0.0, 0.53, 0.74}{0.966} &\textcolor[rgb]{0.0, 0.53, 0.74}{0.994}&\textcolor[rgb]{0.0, 0.53, 0.74}{0.959} &\textcolor[rgb]{0.0, 0.53, 0.74}{0.995}\\
     
         \bottomrule
    \end{tabular}}}
    \vspace{-10pt}
    \label{tab:nuclear}
\end{table*}

\section{Why are the gradients of the entire network not as predictive as that of the last layer?}
We provide two intuitive explanations for this phenomenon, together with pointers to the existing body of work on our learning task, below:
\begin{itemize}
    \item Deep neural networks detect general features at lower (e.g., first) layers and specific features (parts, objects) at higher (e.g., last) layers \citep{chen2021detecting}. Consequently, higher layers would be more sensitive to the changes of distribution shift than lower layers. So, it is reasonable that the last layer can capture more information about the distribution shift.
    \item The gradients are getting weaker and weaker as moving back through the hidden layers \citep{raghu2021vision}. Therefore, the information on the distribution shift contained in the gradients may also decrease after backpropagation through layers.
\end{itemize}

This may explain why the gradients of the last layer can achieve the best for this task. Moreover, the gap between different layers can be affected by network architecture. For example, recent work \citep{raghu2021vision} shows that, compared to CNNs, ViTs have more uniform representations with greater similarity between lower and higher layers. The reason is that most information in ViT passes through skip connections. Therefore, it is natural that using different layers of ViT would not result in a large gap in performance in our method. This will not affect the effectiveness of the proposed method using the last layer.

To certify it, we conduct an ablation study of layer selection with ResNet18 and ViT in Table \ref{tab:res18_layer} and Table \ref{tab:vit_layer}, respectively. Results shown below illustrate that the last layer performs satisfactorily for both ResNet models and transformer-based models. Furthermore, it is worth noting that the estimation performance of ViT is more consistent even with deeper layers compared to ResNet18.

\begin{table}[t]
    \centering
    \caption{Performance comparison between \textsc{GdScore} from different layers on CIFAR-10C with ResNet18. The performance is measured by coefficients of determination (i.e., $R^2$).}
    \begin{tabular}{cccccc}
    
    \toprule
         ResNet18 &Block 1 &Block 2 &Block 3 &Last layer &All\\
         \midrule
         CIFAR-10 &0.880 &0.972	&0.984	&0.986	&0.983\\
    \bottomrule
    \end{tabular}
    
    \label{tab:res18_layer}
\end{table}

\begin{table}[t]
    \centering
    \caption{Performance comparison between \textsc{GdScore} from different layers on CIFAR-10C with ViT. The performance is measured by coefficients of determination (i.e., $R^2$).}
    \begin{tabular}{cccccc}
    \toprule
         ViT &Block 1-18 &Block 9-16 &Block 17-24 &Last layer &All\\
         \midrule
         CIFAR-10  &0.913 &0.953 &0.935 &0.948 &0.938\\
    \bottomrule
    \end{tabular}
    
    \label{tab:vit_layer}
\end{table}

\section{Proofs}
\label{app:proofs}
\subsection{Proof of Theorem~\ref{thm:target-risk-grad-norm}}
\label{app:true-risk-grad-norm}

    We start by proving the following lemma.
    \begin{lem}
    \label{lem:convex-ineq}
        For any convex function $f:\mathbb{R}^D\to\mathbb{R}$ and any $p, q \geq 1$ such that $\frac{1}{p} + \frac{1}{q} = 1$, we have:
        \begin{align*}
        \forall \mbf{a}, \mbf{b} \in \text{dom}(f), \quad 
            |f(\mbf{a})-f(\mbf{b})|\leq \max_{\mbf{c}\in\{\mbf{a},\mbf{b}\}}\{\norm{\nabla f(\mbf{c})}_p\}\cdot \norm{\mbf{a}-\mbf{b}}_q.
        \end{align*}
    \end{lem}
\begin{proof}
    Using the fact that $f$ is convex, we have:
    \begin{align*}
        f(\mbf{a}) - f(\mbf{b}) &\leq \langle \nabla f(\mbf{a}), \mbf{a}-\mbf{b} \rangle \\
        &\leq \lvert \langle \nabla f(\mbf{a}), \mbf{a}-\mbf{b} \rangle \rvert \\ 
        &\leq \sum_{i=1}^p \lvert \nabla f(\mbf{a})_i \left(\mbf{a}_i-\mbf{b}_i \right)\rvert \\
        &\leq \lVert \nabla f(\mbf{a}) \rVert_p \lVert \mbf{a} - \mbf{b} \rVert_q,
    \end{align*}
where we used Hölder's inequality for the last inequality.
The same argument gives:
\begin{equation*}
    f(\mbf{b}) - f(\mbf{a}) \leq \lVert \nabla f(\mbf{b}) \rVert_p \lVert \mbf{b} - \mbf{a} \rVert_q.
\end{equation*}
Using the absolute value, we can combine the two previous results and obtain the desired inequality.
\end{proof}
The proof of Theorem ~\ref{thm:target-risk-grad-norm} follows from applying Lemma~\ref{lem:convex-ineq} to the convex function $\mathcal{L}_T$.

\subsection{Proof of Theorem~\ref{cor:after-one-grad-update}}
\label{app:after-one-grad-update}

\begin{proof} The proof follows from Theorem~\ref{thm:target-risk-grad-norm} by noting that $
        \lVert\bm{\omega}_s - \mbf{c}\rVert_q = \eta \lVert \nabla\mathcal{L}_T(\bm{\omega}_s) \rVert_q $.
\end{proof}

\subsection{Proof of Theorem~\ref{thm:upper_bound_norm_grad}}
\label{app:upper_bound_norm_grad}
We start by introducing some notations. We denote $\mathcal{L}_{\mbf{x}, y}$ the loss evaluated on a specific data-point $(\mbf{x}, y) \sim P_T(\bf{x}, y)$. We can then decompose the expected loss as $\mathcal{L}_T = \mathbb{E}_{P_T(\mbf{x}, y)} \mathcal{L}_{\mbf{x}, y}$. It follows by linearity of the expectation that 
\begin{equation*}
    \nabla  \mathcal{L}_T = \mathbb{E}_{P_T(\mbf{x}, y)} \nabla \mathcal{L}_{\mbf{x}, y}.
\end{equation*}

Then, we prove the following lemma that gives the formulation of the gradient of the cross-entropy.

\begin{lem}
The gradient of the cross-entropy loss with respect to $\bm{\omega}=(\mathbf{w}_k)_{k=1}^K$ writes
\begin{align*}
    \nabla \mathcal{L}_{\mbf{x}, y}(\bm{\omega}) = \left(-y^{(k)} \mathbf{x}(1-\mathrm{s}_{\bm{\omega}}^{(k)}(\mathbf{x})\right)_{k=1}^K.
\end{align*}
\end{lem}

\begin{proof}
    First, let's compute the partial derivative of the softmax w.r.t. $\mathbf{w}_k$ for any $k\in\{1,\dots,K\}$. We have:
    \begin{align*}
        \frac{\partial \mathrm{s}_{\bm{\omega}}^{(k)}(\mbf{x})}{\partial \mbf{w}_k} &= \frac{\mbf{x}e^{\mbf{w}_k^\top \mbf{x}}\left(\sum_{\tilde{k}} e^{\mbf{w}_{\tilde{k}}^\top \mbf{x}}-e^{\mbf{w}_k^\top \mbf{x}}\right)}{\left(\sum_{\tilde{k}} e^{\mathbf{w}_{\tilde{k}}^\top \mbf{x}}\right)^2}\\
        &= \mathbf{x}\left(\mathrm{s}_{\bm{\omega}}^{(k)}(\mathbf{x})-\left[\mathrm{s}_{\bm{\omega}}^{(k)}(\mathbf{x})\right]^2\right).
    \end{align*}
    Using the chain rule, the partial derivative of the loss w.r.t. $\mathbf{w}_k$ writes:
    \begin{align*}
    \frac{\partial\mathcal{L}_{\mbf{x}, y}(\bm{\omega})}{\partial \mathbf{w}_k} &= \frac{\partial\mathcal{L}_{\mbf{x}, y}(\bm{\omega})}{\partial s(\boldsymbol{\omega}, \mathbf{x})} \cdot \frac{\partial s(\boldsymbol{\omega}, \mathbf{x})}{\partial\mathbf{w}_k} \\
    &= 
    \begin{cases}
        - \frac{1}{\mathrm{s}_{\bm{\omega}}^{(k)}(\mathbf{x})} \cdot \mathbf{x}\left(\mathrm{s}_{\bm{\omega}}^{(k)}(\mathbf{x})-\left[\mathrm{s}_{\bm{\omega}}^{(k)}(\mathbf{x})\right]^2\right), &\text{ if } y^{(k)} = 1 \\
        0, &\text{ otherwise }
    \end{cases} \\
    &= -y^{(k)} \mathbf{x}\left(1-\mathrm{s}_{\bm{\omega}}^{(k)}(\mathbf{x}) \right)
    \end{align*}
As the $\frac{\partial\mathcal{L}_{\mbf{x}, y}(\bm{\omega})}{\partial \mathbf{w}_k}$ are the coordinates of $\nabla \mathcal{L}_{\mbf{x}, y}(\bm{\omega})$, we obtain the desired formulation.
\end{proof}

We now proceed to the proof of Theorem~\ref{thm:upper_bound_norm_grad}.

\begin{proof}
    Using the convexity of $\lVert \cdot \rVert _p$ and the Jensen inequality, we have that
    \begin{align*}
        \lVert \nabla \mathcal{L}_T(\bm{\omega}) \rVert_p &= \lVert \mathbb{E}_{P_T(\mbf{x}, y)} \nabla \mathcal{L}_{\mbf{x}, y}(\bm{\omega}) \rVert_p \\
        &\leq \mathbb{E}_{P_T(\mbf{x}, y)} \lVert \mathcal{L}_{\mbf{x}, y}(\bm{\omega}) \rVert _p \tag{Jensen inequality} \\
        &= \mathbb{E}_{P_T(\mbf{x}, y)} \left(\sum_{i=1}^D \sum_{k=1}^K \lvert -y^{(k)} \mathbf{x}_i(1-\mathrm{s}_{\bm{\omega}}^{(k)}(\mathbf{x})\rvert^p\right)^{1/p} \\
        &=\mathbb{E}_{P_T(\mbf{x}, y)} \left(\sum_{k=1}^K y^{(k)} \left(1-\mathrm{s}_{\bm{\omega}}^{(k)}(\mathbf{x})\right)^p\right)^{1/p} \left(\sum_{i=1}^D \lvert \mathbf{x}_i^p \rvert \right)^{1/p}\\
        &=\mathbb{E}_{P_T(\mbf{x}, y)} \left(\left(1-\mathrm{s}_{\bm{\omega}}^{(k_y)}(\mathbf{x})\right)^p\right)^{1/p} \left(\sum_{i=1}^D \lvert \mathbf{x}_i^p \rvert \right)^{1/p} \tag{$k_y$ such that $y^{(k_y)}\!=\!1$}\\
        &=\mathbb{E}_{P_T(\mbf{x}, y)} \alpha(\bm{\omega}, \mbf{x}, y) \lVert \mathbf{x} \rVert_p,
    \end{align*}
where $\alpha(\bm{\omega}, \mbf{x}, y) = \left(1-\mathrm{s}_{\bm{\omega}}^{(k_y)}(\mathbf{x})\right)$, with $k_y$ such that $y^{(k_y)}\!=\!1$. We used the fact that $\mbf{y}$ is a one-hot vector so it has only one nonzero entry.
\end{proof}

\subsection{Proof of Remark~\ref{rmk:p_lower_than_one}}
\label{app:p_lower_than_one}
\begin{proof}
    Using the reverse Minkowski inequality, as $0<p<1$, we have that
    \begin{align*}
        & \lVert \bm{\omega}_s \rVert _p = \lVert \mbf{c} + \eta\cdot \nabla \mathcal{L}_T(\bm{\omega}_s) \rVert _p \geq \lVert \mbf{c} \rVert_p + \eta \cdot \lVert \nabla \mathcal{L}_T(\bm{\omega}_s) \rVert_p \\
        \implies & \lVert \bm{\omega}_s \rVert _p - \lVert \mbf{c} \rVert_p \geq \eta \cdot \lVert \nabla \mathcal{L}_T(\bm{\omega}_s) \rVert_p.
    \end{align*}
In the same fashion, we have that
    \begin{align*}
        & \lVert \mbf{c} \rVert _p = \lVert \bm{\omega}_s - \eta\cdot \nabla \mathcal{L}_T(\bm{\omega}_s) \rVert _p \geq \lVert \bm{\omega}_s \rVert_p + \eta \cdot \lVert \nabla \mathcal{L}_T(\bm{\omega}_s) \rVert_p \\
        \implies & \lVert \mbf{c} \rVert_p - \lVert \bm{\omega}_s \rVert _p \geq \eta \cdot \lVert \nabla \mathcal{L}_T(\bm{\omega}_s) \rVert_p.
    \end{align*}
We obtain the desired upper bound by combining those results.
\end{proof}

\include{checklist}

\end{document}